\documentclass[ijoc,nonblindrev]{informs3} 

\OneAndAHalfSpacedXII 

\usepackage{tikz, relsize, mathtools, cleveref}

\renewcommand{\qed}{\hfill\Halmos}

\newcommand{\RR}{\mathbb{R}}
\newcommand{\NN}{\mathbb{N}}
\newcommand{\st}{\mathrel{}\middle|\mathrel{}}
\newcommand{\abs}[1]{\lvert#1\rvert}
\newcommand{\cupdot}{\mathbin{\mathaccent\cdot\cup}}

\newcommand{\fin}{\mathbf{f}_\mathrm{in}}
\newcommand{\mysin}{\mathbf{s}_\mathrm{in}}
\newcommand{\pin}{\mathbf{p}_\mathrm{in}}
\newcommand{\fout}{\mathbf{f}_\mathrm{out}}
\newcommand{\gin}{\mathbf{g}_\mathrm{in}}
\newcommand{\gout}{\mathbf{g}_\mathrm{out}}
\newcommand{\pstarin}{\mathbf{p}^*_\mathrm{in}}
\newcommand{\pstarout}{\mathbf{p}^*_\mathrm{out}}
\newcommand{\dold}{\mathbf{d}_\mathrm{old}}
\newcommand{\dnew}{\mathbf{d}_\mathrm{new}}

\newcommand{\tikzsize}{\footnotesize}
\newcommand{\tikzscalefactor}{1}

\usetikzlibrary{arrows.meta, shapes}
\tikzset{bigneuron/.style={circle, draw, inner sep = 0, minimum width = 5.5ex}}
\tikzset{smallneuron/.style={circle, draw, inner sep = 0, minimum width = 4.5ex}}
\tikzset{transform/.style={fill=white, circle}}
\tikzset{connection/.style={-{Stealth}}}
\newcommand{\relu}{
	\begin{tikzpicture}
		\draw [line width=1pt] (-1.1ex,0) -- (0,0) -- (0.9ex,0.9ex);
	\end{tikzpicture}
}

\usepackage{natbib}
 \bibpunct[, ]{(}{)}{,}{a}{}{,}%
 \def\bibfont{\small}%
\TheoremsNumberedThrough     

\EquationsNumberedThrough    

\begin{document}


\RUNAUTHOR{Hertrich and Skutella}

\RUNTITLE{Provably Good Knapsack Solutions via Neural Networks of Bounded Size}

\TITLE{Provably Good Solutions to the Knapsack Problem via Neural Networks of Bounded Size\thanks{An extended abstract of this article, including Figures~1--7, appeared in the Proceedings of the AAAI Conference on Artificial Intelligence, volume 35, 7685–7693,~\citep{hertrich2021provably}; see \texttt{https://ojs.aaai.org/index.php/AAAI/article/view/16939}; copyright \copyright~2021, Association for the Advancement of Artificial Intelligence.
}}

\ARTICLEAUTHORS{%
\AUTHOR{Christoph Hertrich}
\AFF{London School of Economics and Political Science, \EMAIL{c.hertrich@lse.ac.uk}}
\AUTHOR{Martin Skutella}
\AFF{Technische Universit{\"a}t Berlin, \EMAIL{martin.skutella@tu-berlin.de}}
} 

\ABSTRACT{%
The development of a satisfying and rigorous mathematical understanding of the performance of neural networks is a major
challenge in artificial intelligence. Against this background, we study the expressive power of neural networks through
the example of the classical NP-hard Knapsack Problem. Our main contribution is a class of recurrent neural networks
(RNNs) with rectified linear units that are iteratively applied to each item of a Knapsack instance and thereby compute
optimal or provably good solution values. We show that an RNN of depth four and width depending quadratically on the
profit of an optimum Knapsack solution is sufficient to find optimum Knapsack solutions. We also prove the following
tradeoff between the size of an RNN and the quality of the computed Knapsack solution: for Knapsack instances consisting
of~$n$ items, an RNN of depth five and width~$w$ computes a solution of value at
least~\mbox{$1-\mathcal{O}(n^2/\sqrt{w})$} times the optimum solution value. Our results build upon a classical dynamic
programming formulation of the Knapsack Problem as well as a careful rounding of profit values that are also at the core
of the well-known fully polynomial-time approximation scheme for the Knapsack Problem.
A carefully conducted computational study qualitatively supports our theoretical size
bounds. Finally, we point out that our results can be generalized to many other combinatorial optimization problems that
admit dynamic programming solution methods, such as various Shortest Path Problems, the Longest Common Subsequence
Problem, and the Traveling Salesperson Problem.
}%



\maketitle

%


\section{Introduction}

Deep learning and neural networks (NNs) are at the heart of some of the greatest advances in modern computer science. They enable huge breakthroughs in applications like computer vision, translation, speech recognition, and
autonomous driving, to name just a few; see, e.g., \citep{LeCunBengioHinton:DeepLearning}.
%
While numerous computational studies
present impressive empirical proof of neural networks' computational power, we are still far away from a more
rigorous theoretical explanation of these observations.

Apart from the popular applications named above, it has been shown that NNs have high potential for practically solving combinatorial optimization (CO) problems or empirically improving classical solution methods \citep{BengioLodiProuvost:MLforCO}. For example, \citet{yang2018boosting} and \citet{xu2020deep} utilize NNs in order to empirically enhance \emph{dynamic programming}, a very classical CO method.
While the methods used in these papers indeed provide fast and empirically near-optimal solutions, their use of NNs makes it virtually impossible to give theoretical optimality or worst-case approximation guarantees.
Motivated by this imbalance, and focusing on the Knapsack Problem, which is a prime example of CO problems that can be solved via dynamic programming, we investigate the following question:\vspace{0.5em}

\begin{center}
	\emph{Which neural network size is theoretically sufficient\\to find solutions of provable quality for the Knapsack Problem?}
\end{center}\vspace{0.5em}

We give an answer to this question by presenting a class of carefully constructed NNs with provable quality guarantees and support our size bounds by a computational study. Finally, we argue that our approach is not at all specific for the Knapsack Problem, but can be generalized to many other CO problems, e.g., various Shortest Path Problems, the Longest Common Subsequence
Problem, and the Traveling Salesperson Problem.

It is worth to note that it is by no means obvious that classical algorithms can be translated to neural networks with bounded size. In particular, as discussed by \citet{hertrich2021relu}, even for efficiently (polynomial-time) solvable problems it is a priori unclear whether polynomially sized neural networks can compute provably good solutions.

\subsection*{The Knapsack Problem} 
The Knapsack Problem constitutes one of the oldest and most studied problems
in Combinatorial Optimization (CO). Given a set of items with certain profit and size values, as well as a Knapsack capacity, the Knapsack Problem asks for a subset of items with maximum total profit such that the total size of the subset does not exceed the capacity.

The Knapsack Problem is one of Karp's 21 original NP-complete problems \citep{Karp:21Problems} and has numerous applications in a
wide variety of fields, ranging from production and transportation, over finance and investment, to network
security and cryptography. It often appears as a subproblem at the core of more complex problems; see,
e.g.,~\citep{KellererPferschyPisinger2004, MartelloToth1990}. This fact substantiates the Knapsack Problem's
prominent importance as one of the key problems in~CO.
In particular, the Knapsack
Problem is frequently used as a testbed for measuring the progress of various exact and heuristic
solution approaches and computational methods such as, e.g., integer programming, constraint programming, or
evolutionary algorithms. In integer programming, for example, the Knapsack Problem and so-called `Knapsack
Inequalities' play a central role, both with respect to theory as well as in the development of modern
computational methods; see, e.g.,~\citep{BertsimasWeismantel05,FischettiLodi2010}. The Knapsack Problem is
therefore a natural and important object of study in order to advance our theoretical understanding of neural
networks and get closer to a rigorous explanation of their stunning success in so many applications, including miscellaneous optimization problems.

\subsection*{Related work}

The idea of using neural networks (NNs) to practically solve CO problems became popular with the
work of \citet{hopfield1985neural}. Hopfield NNs are special versions of
recurrent neural networks (RNNs) that find solutions to optimization problems by converging towards a minimum
of an energy function. \citet{Smith:NNforCOreview} reviews this early stream of research. While most authors
mainly focus on the Traveling Salesperson Problem (TSP), \citet{ohlsson1993neural} study a so-called mean field
NN for (generalizations of) the Knapsack Problem and empirically assess the quality of its solutions.

While there has been less research at the intersection of CO and NNs in the 2000s, modern advances in the area of
deep learning have boosted the interest in this direction again; see \citep{BengioLodiProuvost:MLforCO} for a general review and \citep{cappart2021combinatorial} for a focused review on graph neural networks. Common applications include speeding up solvers for
mixed-integer linear programs, for instance, by automatically learning on which variables to branch in
branch-and-bound algorithms; see \citep{lodi2017learning} for a survey. 
Machine learning has also been applied to modeling aspects of CO, as reviewed by \citet{lombardi2018boosting}, and to several
specific CO problems, where the TSP is often one of them \citep{Bello:NeuralCOwithRL, Emami:learningPermutations,	Khalil:LearningCOoverGraphs, kool2019attention, nowak:quadrAssignment, Vinyals:PointerNetworks}. The
different methods used by these authors include feedforward and recurrent neural networks, reinforcement
learning, attention mechanisms, pointer networks, graph embeddings, and graph neural networks. Expressivity results similar to the paper at hand, but using completely different techniques, have been achieved for the Maximum Flow Problem by \citet{hertrich2021relu}.

There have also been specific applications of neural networks to the Knapsack Problem.
For example,
\citet{Bello:NeuralCOwithRL} utilize an RNN trained by reinforcement learning and \citet{Gu2018pointer} use a pointer network to find empirically good knapsack solutions. \citet{li2021novel} derive heuristics inspired by game theory to solve a non-linear knapsack version whose objective function is given by a neural network.

Particularly related to our work are interactions between neural networks and dynamic programming algorithms. For example, \citet{yang2018boosting} and \citet{xu2020deep} use NNs to speed up
dynamic programming algorithms for CO problems. The key difference to our work, however, is that NNs are used as heuristics in these papers, making it virtually impossible to give any meaningful worst-case performance guarantees. Another interesting research stream deals with the learnability of algorithms. In this context, \citet{xu2020can} develop the concept of \emph{algorithmic alignment} and show that dynamic programming algorithms align well with graph neural networks. \citet{velivckovic2020neural} perform a more empirical study in this direction. These results concentrate on learnability in contrast to our focus on expressivity. Still, they agree with the message of our paper that dynamic programming is a good paradigm for bringing classical algorithms and neural networks closer together.

The recent success of deep neural networks has also triggered a lot of research on their general expressivity. As
we do in this paper, many authors focus on the simple but practically powerful model of feedforward NNs with
activations in the form of rectified linear units (ReLU). Since \citet{glorot2011deep} corroborated their empirical
success, such ReLU NNs have been
established as a standard model in machine learning within the past decade. ReLU NNs can compute any continuous piecewise linear function~\citep{Arora:DNNwithReLU, goodfellow2013maxout}. This fact implies universal approximation properties. 
Recent progress concerning the classes of piecewise linear functions representable by ReLU neural networks with a certain architecture was made by \citet{hertrich2021towards} and \citet{haase2023lower}.
Results on the exact expressivity boosted success with respect to understanding the computational complexity of the training problem for ReLU networks~\citep{Arora:DNNwithReLU, bertschinger2022training, froese2022computational, GKMR21, khalife2022neural}.
A variety of results has been achieved on depth vs.\ width tradeoffs \citep{Arora:DNNwithReLU, eldan2016power, hanin2019universal, hanin2017approximating, liang2017deep, nguyen2018neural, safran2017depth, Telgarsky15, telgarsky2016benefits, yarotsky2017error}. Closely related are
investigations concerning the number and structure of linear regions that NNs with certain size and depth may
have \citep{hanin2019complexity, montufar2014regions, pascanu2014number, raghu2017expressive}.
\citet{serra2018bounding} use mixed-integer programming for precisely counting the number of such regions.
\citet{mukherjee2017lower} prove size lower bounds to represent Boolean functions with NNs of limited depth.

\subsection*{Our contribution} 
We present a rigorous mathematical study on the expressivity of NNs through the example of the NP-hard Knapsack Problem. To this end, we show that there is a class of feedforward ReLU NNs of bounded size that compute \emph{provably} good solutions to the
NP-hard Knapsack Problem. In Section~\ref{Sec:DPNN}, we first present such an NN of depth~$\mathcal{O}(n)$ and
width~$\mathcal{O}((p^*)^2)$ that always finds the exact value of an optimum Knapsack solution. Here, $n$ is the number of items in the Knapsack instance, and $p^*$ is an a priori known upper bound on the value of an optimum solution. More
precisely, the optimum solution value is found by iteratively applying an RNN of depth four and
width~$\mathcal{O}((p^*)^2)$ to the~$n$ items of a Knapsack instance. As~$p^*$ can, e.g., be chosen as the total size of all items, the RNN's
width is pseudo-polynomially bounded in the input size of the Knapsack instance. Due to the Knapsack Problem's NP-hardness, however, there is no polynomial-size NN with polynomial-time computable weights that always finds the optimum solution value, unless P$\,=\,$NP; compare the discussion in the next subsection. 

In Section~\ref{Sec:FPTASNN}, we prove that the width of the NNs can be drastically decreased while still obtaining solution values
of provable quality in the worst case. We construct an RNN of depth five and fixed width~$w$ which, when applied iteratively to the~$n$ items of a Knapsack instance, always produces a solution value of at least $1-\mathcal{O}(n^2/\sqrt{w})$ times the optimum solution value. In particular, an~\mbox{$\varepsilon$-approximate} solution value can be guaranteed by choosing width~$w\in\mathcal{O}(n^4/\varepsilon^2)$. 
If we require the weights of the neural network to be polynomial-time computable, then the dependence of the width on~$\varepsilon$ is unavoidable, unless P$\,=\,$NP; compare the discussion in the next subsection.
To the best of our knowledge, our results establish the first rigorous tradeoff between the size of neural networks for CO problems and their worst-case solution quality. 

Even though we cannot show theoretical lower bounds beyond what is directly implied by NP-hardness, we provide empirical evidence for the superlinear dependence on~$p^*$ (and~$1/\varepsilon$) in Section~\ref{Sec:Experiments}.

The idea behind our construction of the NNs is to mimic the classical dynamic program for the Knapsack
Problem. More precisely, the output neurons of the RNN can be seen as elements of the dynamic programming state
space while the hidden neurons and the network itself implement the recursive dynamic programming
formula. Here, the main technical difficulty is to always filter out the correct entries of the previous state
space (input neurons) needed in the recursive formula. In addition, our NNs of fixed width rely on a subtle
variant of the rounding procedure that turns the pseudo-polynomial dynamic program into a
fully polynomial-time approximation scheme for the Knapsack Problem.

In this paper, the Knapsack Problem mainly serves as a prominent showcase for a novel approach to the rigorous analysis
of neural networks' expressivity. This approach is by no means specific for the Knapsack Problem. In
Section~\ref{Sec:OtherProblems}, we discuss how it can be applied to NNs for other combinatorial optimization problems
that can be solved via dynamic programming. In particular, we establish similar results for the Longest Common
Subsequence Problem, the Single-Source and All-Pairs Shortest Path Problems, as well as the NP-hard Traveling
Salesperson Problem and the Constrained Shortest Path Problem. For the latter problem one can show similar results on
the tradeoff between the size of NNs and their solution quality.

\subsection*{Uniformity and Lower Bounds Implied by NP-Hardness}

As mentioned earlier, NP-hardness makes it unlikely that we can get rid of the dependence on $p^*$ (in the exact case) and $\varepsilon$ (in the approximate case) for our constructions. However, strictly speaking, NP-hardness does not directly imply conditional lower bounds on the size of a neural network. Similar to the complexity of Boolean circuits~\citep{arora2009computational}, one needs an additional \emph{uniformity} assumption to deduce lower bounds conditioned on P$\,\neq\,$NP, which we will explain now.

We call a construction of neural networks \emph{uniform}, if one can compute the weights of the neural network in polynomial time. Having a uniform neural network construction of polynomial size to solve a certain problem implies that the respective problem can be solved in polynomial time: for a given instance, one first computes the neural network and then executes it to compute a solution to the instance. Hence, a uniform neural network construction of polynomial size for the Knapsack Problem cannot exist unless P$\,=\,$NP. This implies that the dependences on $p^*$ (in the exact case) and $\varepsilon$ (in the approximate case) are unavoidable if one aims for uniform constructions. Note that the constructions provided in this paper are trivially uniform because we explicitly provide the weights of the neural networks.

Without the uniformity assumption, however, it might be possible that P$\,\neq\,$NP is true and there still exist polynomial-size neural networks for the Knapsack Problem. We might just be unable to compute these networks (in polynomial time or at all), which is why they cannot be used to obtain polynomial-time algorithms and thus obtain a contradiction to P$\,\neq\,$NP.

\section{Preliminaries}\label{Sec:Prelim}

\subsection*{Neural networks with rectified linear units} We use definitions and notations similar to
\citep[Chapter~20]{Shalev2014:UnderstandingML}. A \emph{feedforward neural network with rectified linear
	units}, abbreviated by ReLU NN, or simply NN, is a finite, directed, acyclic graph $(V,E)$, equipped with arc
weights $w_{uv}\in\RR$, for each $(u,v)\in E$, and node biases $b_v\in\RR$, for each node~\mbox{$v\in V\setminus
	V_0$}. Here, $V_0$ is the set of nodes with in-degree zero. The nodes in $V$ are called \emph{neurons}. The
\emph{depth}~$k$ is the length of a longest path in the graph. In the following we suppose that neurons are
grouped into \emph{layers} $V=V_0\cupdot V_1 \cupdot \cdots \cupdot V_k$ such that the layer index strictly
increases along each arc.\footnote{Some authors only allow connections between successive layers. One can
	create such a structure by adding additional neurons propagating the values of neurons from former layers
	through the network. For our purposes, however, it is convenient to omit this restriction.} Further, we assume that
$V_0$ and~$V_k$ are precisely the sets of neurons with in-degree and out-degree zero, respectively.
Consequently, they are called \emph{input neurons} and \emph{output neurons}, respectively. Neurons in
$V\setminus(V_0\cup V_k)$ are called \emph{hidden neurons}. Let $n_\ell=\abs{V_\ell}$ be the number of neurons in the
$\ell$-th layer. The \emph{width} and \emph{size} of the NN are defined to be $\max\{n_1,\dots,n_{k-1}\}$ and
$\sum_{\ell=1}^{k-1} n_\ell$, respectively.

Every NN computes a function $\RR^{n_0}\to \RR^{n_k}$ as follows. Given an input vector
$x\in\RR^{n_0}$, we associate an activation $a(v)$ with every neuron~\mbox{$v\in V\setminus V_0$} and an output $o(v)$
with every neuron~\mbox{$v\in V\setminus V_k$}. First, the output values $o(v)$, \mbox{$v\in V_0$,} of the $n_0$ input
neurons equal the $n_0$ components of input vector $x$. Second, the activation of a neuron~$v\in
V\setminus V_0$ is the weighted sum of outputs of all predecessors plus its bias, that is, $a(v)=b_v +
\sum_{u\colon(u,v)\in E} w_{uv} o(u)$. Third, for each hidden neuron $v\in V\setminus(V_0\cup V_k)$, the output
is determined by $o(v)=\sigma(a(v))$, where $\sigma$ is the so-called \emph{activation function}. In this paper,
$\sigma$ is always the \emph{rectifier function} $\sigma(z)=\max\{0,z\}$ applied pointwise. Neurons having this activation function are
called \emph{rectified linear units} (ReLUs). Finally, the output vector $y\in \RR^{n_k}$ consists of the $n_k$
activation values $a(v)$ of the $n_k$ output neurons~$v\in V_k$.
Figure~\ref{Fig:Min2Num} gives an example, which will also be used as a subnetwork in later sections.
\begin{figure}[t]
	\centering
	\begin{tikzpicture}[scale=\tikzscalefactor, every node/.style={transform shape}]\tikzsize
		\node[smallneuron] (x1) at (0,5ex) {$x_1$};
		\node[smallneuron] (x2) at (0,0) {$x_2$};
		\node[smallneuron] (n1) at (11ex,5ex) {\relu};
		\node[smallneuron] (y) at (22ex,0ex) {$y$};
		\scriptsize
		\draw[connection] (x1) -- node[above, pos=0.4]{-1} (n1);
		\draw[connection] (x2) -- node[above, pos=0.4]{1} (n1);
		\draw[connection] (n1) -- node[above, pos=0.5]{-1} (y);
		\draw[connection] (x2) -- node[above, pos=0.65]{1} (y);
		
	\end{tikzpicture}
	\caption{An NN with two input neurons, labeled $x_1$ and $x_2$, one
		hidden neuron, labeled with the shape of the rectifier function, and one output neuron, labeled $y$. The arcs
		are labeled with their weights and all biases are zero. The network has depth 2, width 1, and size 1. It
		computes the function \[x \mapsto y= x_2-\max\{0,x_2-x_1\}
		=\min\{x_1,x_2\}.\]}
	\label{Fig:Min2Num}
\end{figure}
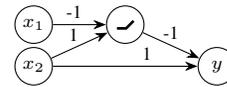

Since feedforward NNs have a fixed input size, a common way of handling sequential
inputs of arbitrary length is to use \emph{recurrent neural networks} (RNNs).
This type of NNs has become very popular,
e.g., for tasks in language~\citep{sutskever2014sequence, bahdanau2015neural} or speech processing~\citep{graves2013speech}.
Essentially, an RNN is a feedforward NN
that is used repeatedly for every piece of the input sequence and maintains a hidden state by passing (part of)
its output in each step as an additional input to the next step. More precisely, in the $i$-th step, the input of the RNN consists of the $i$-th input vector $x_{i}$, as well as, the previous hidden state vector $h_{i-1}$. In the same manner as a feedforward NN described above, it then computes the $i$-th output vector $y_i$, as well as, the new hidden state vector $h_i$. The basic structure of an RNN is shown in Figure~\ref{Fig:RNN}. Sometimes it holds that $y_i=h_i$, that is, the $i$-th output is actually equal to the $i$-th hidden state.

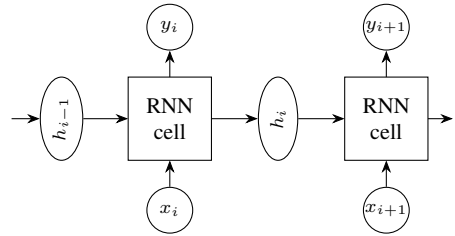
\begin{figure}[t]
	\centering
	\begin{tikzpicture}[scale=\tikzscalefactor, every node/.style={transform shape}]\tikzsize
		\node[draw, ellipse, rotate=90, minimum width = 10ex] (minus1) at (0,0) {$h_{i-1}$};
		\node[draw, ellipse, rotate=90, minimum width = 10ex] (null) at (26ex,0) {$h_i$};
		\draw [] (8ex, -5ex) rectangle (18ex, 5ex)  node [pos=0.5, align=center]{\larger RNN\\\larger cell};
		\draw [] (34ex, -5ex) rectangle (44ex, 5ex)  node [pos=0.5, align=center]{\larger RNN\\\larger cell};
		\node[bigneuron] (x0) at (13ex,-11ex) {$x_{i}$};
		\node[bigneuron] (x1) at (39ex,-11ex) {$x_{i+1}$};
		\node[bigneuron] (y0) at (13ex,11ex) {$y_{i}$};
		\node[bigneuron] (y1) at (39ex,11ex) {$y_{i+1}$};
		\draw[connection] (minus1) -- (8ex,0);
		\draw[connection] (null) -- (34ex,0);
		\draw[connection] (-6ex,0) -- (minus1);
		\draw[connection] (18ex,0) -- (null);
		\draw[connection] (44ex,0) -- (47ex,0);
		\draw[connection] (x0) -- (13ex,-5ex);
		\draw[connection] (x1) -- (39ex,-5ex);
		\draw[connection] (13ex,5ex) -- (y0);
		\draw[connection] (39ex,5ex) -- (y1);
	\end{tikzpicture}
	\caption{Basic structure of an (unfolded) RNN.}
	\label{Fig:RNN}
\end{figure}

\subsection*{Notations and Algorithms for the Knapsack Problem}
An instance of the Knapsack Problem consists of $n$ items $1,2,\dots,n$,
where each item~$i\in[n]$ comes with a given profit~$p_i\in\NN$ and size~$s_i\in\left]0,1\right]$, together
with a Knapsack that can hold any subset~$M\subseteq[n]$ of items of total size~$\sum_{i\in M}s_i$ at most~$1$.
The task is to find such a subset~$M\subseteq[n]$ that maximizes the total profit~$\sum_{i\in M}p_i$. Here and
in the following, we use $\NN\coloneqq\{1,2,3,\dots\}$ to denote the natural numbers (without zero), and for 
every~$k\in\NN$, we let~$[k]\coloneqq\{1,2,\dots,k\}$.

We outline a classical dynamic programming formulation for the Knapsack Problem.
Let $p^*$ be an upper bound on the optimum solution value, e.g., $p^*=\sum_{i=1}^{n} p_i$. For~$i\in[n]$ and~$p\in[p^*]$, let
\[
f(p,i) \coloneqq \min \left\{\sum\nolimits_{j\in M} s_j \st M\subseteq[i],~\sum\nolimits_{j\in M} p_j \geq p \right\}
\]
be the minimum size of a subset of the first $i$ items with total profit at least $p$. With $f(p,i)\coloneqq0$
for~$p\leq 0$ and $f(p,0)\coloneqq+\infty$ for $p\in[p^*]$, the values of $f$ can be computed recursively by
\begin{align}
	f(p,i)=\min\bigl\{f(p,i-1), f(p-p_i,i-1)+s_i\bigr\}\label{eq:recursion}
\end{align}
for $i\in[n]$, $p\in[p^*]$, where the first option corresponds to not using the $i$-th item, while the second option corresponds to using it.
The optimum solution value is \mbox{$\max\{p\in[p^*]\mid f(p,n)\leq 1\}$}, and the optimum subset can easily be found
by backtracking. The runtime of the dynamic program is $\mathcal{O}(np^*)$, thus pseudo-polynomial in the
input size.

Due to NP-hardness of the Knapsack Problem, one cannot expect to find an exact algorithm with polynomial running time. However, by carefully downscaling and
rounding the profit values in the dynamic program, one can obtain a \emph{fully polynomial-time approximation scheme} (FPTAS). That is, for each $\varepsilon>0$, one can compute a feasible solution with guaranteed profit of at least $1-\varepsilon$ times the optimal profit, with running time polynomial in the input size and $1/\varepsilon$. For more details, we refer to the books by \citet{Hochbaum:ApproxAlgs}, \citet{Vazirani:ApproxAlgs}, or \citet{williamson2011design}.

Usually, the Knapsack Problem is defined with integer size values $s_i\in\NN$ and some Knapsack capacity~$S\in\NN$, bounding the total size of chosen items. Dividing all item sizes by~$S$ transforms such an instance into an instance of the type considered here. For the case of integral item sizes, there is also a pseudo-polynomial dynamic programming formulation parameterized by the size instead of the profit values; see, e.g., \citep[Section~6.4]{kleinberg2006algorithm}. Our construction in Section~\ref{Sec:DPNN} can analogously be applied to this formulation. This variant, however, does not easily extend to an FPTAS. We therefore stick to the variant parametrized by the profit values as introduced above.

\section{An Exact RNN for the Knapsack Problem}\label{Sec:DPNN}
In this section we introduce the \emph{DP-NN}, an NN that exactly executes the dynamic program described in Section~\ref{Sec:Prelim}. In fact, the DP-NN is an RNN that receives the items one by one and computes the state space of the dynamic program for the items seen so far. In the first subsection we explain the high-level idea, before we provide further details in the second subsection.

\subsection*{High-level idea of the construction}

Like the dynamic program described in Section~\ref{Sec:Prelim}, the DP-NN requires a fixed upper bound~$p^*$ on the optimal objective value of the Knapsack Problem. We relax this condition in Section~\ref{Sec:FPTASNN}, when we investigate how the FPTAS for the Knapsack Problem can be implemented as an~NN.

In the $i$-th step, the DP-NN receives $p^* + 2$ inputs, namely $f(p,i-1)$ for $p\in[p^*]$, as well as~$p_{i}$ and $s_{i}$. It computes $p^*$ output values, namely $f(p,i)$ for $p\in[p^*]$. Hence, overall it has $p^*+2$ input neurons and $p^*$ output neurons. Figure~\ref{Fig:RNNStructure} illustrates the recurrent structure of the NN, which computes the state space of the dynamic program.
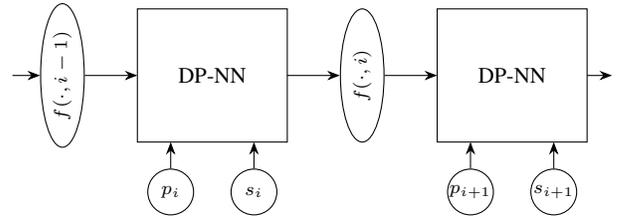
\begin{figure}[ht]
	\centering
	\begin{tikzpicture}[scale=\tikzscalefactor, every node/.style={transform shape}]\tikzsize
		\node[draw, ellipse, rotate=90, minimum width = 16ex] (minus1) at (0,0) {$f(\cdot,i-1)$};
		\node[draw, ellipse, rotate=90, minimum width = 16ex] (null) at (36ex,0) {$f(\cdot,i)$};
		\draw [] (9ex, -8ex) rectangle (27ex, 8ex)  node [pos=0.5]{\larger DP-NN};
		\draw [] (45ex, -8ex) rectangle (63ex, 8ex)  node [pos=0.5]{\larger DP-NN};
		\node[bigneuron] (p0) at (13ex,-14ex) {$p_{i}$};
		\node[bigneuron] (s0) at (23ex,-14ex) {$s_{i}$};
		\node[bigneuron] (p1) at (49ex,-14ex) {$p_{i+1}$};
		\node[bigneuron] (s1) at (59ex,-14ex) {$s_{i+1}$};
		\draw[connection] (minus1) -- (9ex,0);
		\draw[connection] (null) -- (45ex,0);
		\draw[connection] (-6ex,0) -- (minus1);
		\draw[connection] (27ex,0) -- (null);
		\draw[connection] (63ex,0) -- (66ex,0);
		\draw[connection] (p0) -- (13ex,-8ex);
		\draw[connection] (s0) -- (23ex,-8ex);
		\draw[connection] (p1) -- (49ex,-8ex);
		\draw[connection] (s1) -- (59ex,-8ex);
	\end{tikzpicture}
	\caption{Recurrent structure of the DP-NN to solve the Knapsack Problem.}
	\label{Fig:RNNStructure}
\end{figure}

In the following it is very important to distinguish fixed parameters of the NN from activation and output values of neurons that depend on the particular Knapsack instance. We denote the latter by bold symbols in order to make the difference visible. Moreover, in order to make the recurrent structure of our NN obvious, we do not use the index $i$ in the following description of the network. Instead, we denote the $n_0=p^*+2$ input values by $\fin(p)$ for $p\in[p^*]$, as well as $\pin$ and $\mysin$. The~$p^*$ output values are denoted by $\fout(p)$ for $p\in[p^*]$.
The goal is to implement the recursion
\[
\fout(p)=\min\bigl\{\fin(p),~\fin(p-\pin)+\mysin\bigr\}\quad\text{for $p\in[p^*]$}
\]
in an NN; cp.~\eqref{eq:recursion}. It consists of an addition and taking a minimum, which are both simple operations for an NN. Hence, ideally, we would like to have an architecture as depicted in Figure~\ref{Fig:NonRealizableNN} for computing~$\fout(p)$ for every $p\in[p^*]$.
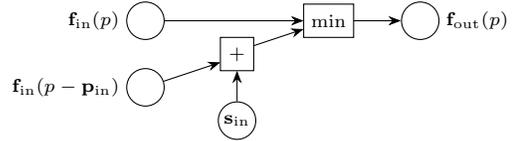
\begin{figure}[t]
	\centering
	\begin{tikzpicture}[scale=\tikzscalefactor, every node/.style={transform shape}]\tikzsize
		\node[smallneuron, label=right:{$\fout(p)$}] (out) at (44ex,0) {};
		\node[smallneuron, label=left:{$\fin(p)$}] (in) at (11ex,0) {};
		\node[smallneuron, label=left:{$\fin(p-\pin)$}] (in2) at (11ex,-8ex) {};
		\node[smallneuron] (si) at (22ex, -12ex) {$\mysin$};
		\draw (30ex,-2ex) rectangle (36ex,2ex) node [pos=0.5]{$\min$};
		\draw (20ex,-6ex) rectangle (24ex,-2ex) node [pos=0.5]{$+$};
		\draw[connection] (in2) -- (20ex,-5ex);
		\draw[connection] (si) -- (22ex, -6ex);
		\draw[connection] (24ex,-3ex) -- (30ex,-1ex);
		\draw[connection] (in) -- (30ex,0);
		\draw[connection] (36ex,0) -- (out);
	\end{tikzpicture}
	\caption{Desirable architecture for computing $\fout(p)$, $p\in[p^*]$, from the inputs. However, the existence of an edge (nonzero weight) depends critically on the input value~$\pin$, which is not allowed.}
	\label{Fig:NonRealizableNN}
\end{figure}
The problem with this is, however, that the decision which component of $\fin$ is accessed in order to compute the sum with $\mysin$ depends on the input value~$\pin$. Since we aim for an architecture that is fixed and works for general input values~$\pin$, we have to extend our construction as depicted in Figure~\ref{Fig:HighLevelNN}.
\begin{figure}[t]
	\centering
	\begin{tikzpicture}[scale=\tikzscalefactor, every node/.style={transform shape}]\tikzsize
		\node[smallneuron, label=right:{$\fout(p)$}] (out) at (44ex,0) {};
		\node[smallneuron, label=left:{$\fin(p)$}] (in) at (0,0) {};
		\node[smallneuron] (si) at (22ex, -24ex) {$\mysin$};
		\draw (30ex,-2ex) rectangle (36ex,2ex) node [pos=0.5]{$\min$};
		\draw (20ex,-6ex) rectangle (24ex,-2ex) node [pos=0.5]{$+$};
		\draw[connection] (si) -- (22ex, -6ex);
		\draw[connection] (24ex,-3ex) -- (30ex,-1ex);
		\draw[connection] (in) -- (30ex,0);
		\draw[connection] (36ex,0) -- (out);
		
		\node[smallneuron] (pi) at (11ex, -24ex) {$\pin$};
		\node[smallneuron, label=left:{$\fin(p-p')$}] (in2neu) at (0,-15ex) {};
		\node[smallneuron] (in3) at (0,-10ex) {};
		\node[smallneuron] (in4) at (0,-20ex) {};
		\draw[connection] (in3) -- (20ex,-4.5ex);
		\draw[connection] (in2neu) -- (20ex,-5ex);
		\draw[connection] (in4) -- (20ex,-5.5ex);
		\draw[fill=white] (6ex,-18ex) rectangle (16ex,-4ex) node [pos=0.5]{$p'\stackrel{?}{=}\pin$};
		\draw[connection] (pi) -- (11ex,-18ex);
	\end{tikzpicture}
	\caption{High-level idea how the DP-NN computes $\fout(p)$ for $p\in[p^*]$ from the inputs.}
	\label{Fig:HighLevelNN}
\end{figure}
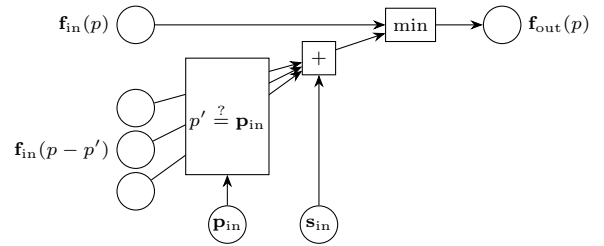
As we do not know the value of $\pin$ in advance, we connect every input neuron $\fin(p-p')$, $p'\in[p-1]$, to the unit that computes the sum $\fin(p-\pin)+\mysin$. Since we only want to take the value $\fin(p-\pin)$ into account, we need to add an additional unit that disables those connections if $p'\neq\pin$.

Due to the integrality of the profit values, this additional unit can be realized with two hidden layers and a constant number of neurons for every value of $p\in[p^*]$ and $p'\in[p-1]$, as we show in a moment. Computing the minimum adds a third hidden layer. Hence, the DP-NN has depth four while width and size are in $\mathcal{O}((p^*)^2)$.
Unfolding the RNN and viewing it as a single feedforward NN executing the whole dynamic program results in depth $\mathcal{O}(n)$ and size $\mathcal{O}(n(p^*)^2)$.
In the next subsection, we provide a detailed construction of the DP-NN and prove the following theorem.

\begin{theorem}\label{Thm:DPCorrectness}
	For a Knapsack instance with capacity $S=1$, \mbox{$s_i\in\left]0,1\right]$}, and~$p_i\in\NN$, for $i\in[n]$, with an upper bound $p^*$ on the optimal solution value, the corresponding dynamic programming values $f(p,i)$, $i\in[n]$, $p\in p^*$, can be exactly computed by iteratively applying the DP-NN $n$ times.
\end{theorem}

As discussed in the introduction, due to the NP-hardness of the Knapsack Problem, the dependence on $p^*$ cannot be avoided if one aims for exact, uniform neural network constructions, unless P$\,=\,$NP.

\subsection*{Details of the construction and correctness}

Note that for size values larger than the Knapsack capacity, which is equal to $1$ by our definition, we do not really care how large they actually are. Therefore, we define
$
\tilde{f}(p,i)=\min\{f(p,i), 2\}
$
to be the values of the dynamic program truncated at $2$. In other words, we replace all values in the interval $[2,+\infty]$ by $2$. Note that the recursion
\begin{equation}\label{Equ:RecursionTilde}
	\tilde{f}(p,i)=\min\bigl\{\tilde{f}(p,i-1), \tilde{f}(p-p_i,i-1)+s_i\bigr\}
\end{equation}
is still valid with starting values $\tilde{f}(p,i)=0$ for $p\leq 0$ and $\tilde{f}(p,0)=2$ for $p\in[p^*]$. Instead of computing the actual values of $f$, the DP-NN computes the values of $\tilde{f}$.

The DP-NN has three hidden layers. After the $n_0=p^*+2$ input neurons $\fin(p)$ for $p\in[p^*]$,~$\pin$, and $\mysin$, the first hidden layer consists of $n_1=2p^*$ neurons whose outputs are denoted by $\mathbf{o}_1^{(+)}(k)$ and $\mathbf{o}_1^{(-)}(k)$ for $k\in [p^*]$. Its role is to detect whether $k=\pin$. If yes, then both $\mathbf{o}_1^{(+)}(k)$ and~$\mathbf{o}_1^{(-)}(k)$ should be zero, otherwise at least one of them should be large (i.e., at least 2). In the second hidden layer, we have $n_2=p^*(p^*-1)/2$ neurons, denoted by $\mathbf{o}_2(p,k)$ for $p\in[p^*]$ and $k\in[p-1]$. A neuron in this layer should output $\fin(p-\pin)$ if $k=\pin$ and zero otherwise. This way, the sum~$\sum_{k=1}^{p-1} \mathbf{o}_2(p,k)$ equals $\fin(p-\pin)$. The third hidden layer has $n_3=p^*$ neurons, denoted by~$\mathbf{o}_3(p)$ for $p\in[p^*]$. It is used for computing the minimum of $\fin(p)$ and $\mysin+\fin(p-\pin)$ using several copies of the construction shown in Figure~\ref{Fig:Min2Num}. Finally, the $n_4=p^*$ output values are denoted by $\fout(p)$ for $p\in[p^*]$.
The following equations define the DP-NN.
\begin{subequations}\label{Eq:DPNN}
	\begin{align}
		\mathbf{o}_1^{(+)}(k) &= \sigma(2(\pin-k)),&&k\in[p^*],\\
		\mathbf{o}_1^{(-)}(k) &= \sigma(2(k-\pin)),&&k\in[p^*],\\
		\mathbf{o}_2(p,k) &= \sigma(\fin(p-k)-\mathbf{o}_1^{(+)}(k)-\mathbf{o}_1^{(-)}(k)),&&p\in[p^*],~k\in[p-1],\\
		\mathbf{o}_3(p) &= \sigma\left(\fin(p)-\left(\mysin+\sum\nolimits_{k=1}^{p-1} \mathbf{o}_2(p,k)\right)\right),&&p\in[p^*],\\
		\fout(p) & = \fin(p)-\mathbf{o}_3(p),&&p\in[p^*].
	\end{align}
\end{subequations}

Our next goal is to prove Theorem~\ref{Thm:DPCorrectness}, which states that the DP-NN indeed solves the Knapsack Problem exactly. We do so by going through the NN layer by layer and show what the individual layers do.

As mentioned, the role of the first hidden layer is to detect the input value of $\pin$ and to provide a large value for every $k$ that is not equal to $\pin$. The following lemma follows immediately from the construction and the properties of the rectifier function $\sigma$.

\begin{lemma}\label{Lem:lay1}
	Let $\pin\in\NN$. Then, for every $k\in[p^*]$, it holds that $\mathbf{o}_1^{(+)}(k)+\mathbf{o}_1^{(-)}(k)=0$ if and only if $k=\pin$, and $\mathbf{o}_1^{(+)}(k)+\mathbf{o}_1^{(-)}(k)\geq 2$ otherwise.
\end{lemma}

The role of the second layer is to compute $\fin(p-\pin)$, which is needed in the dynamic program. The main difficulty of this step is that it depends on the input $\pin$ which neuron to access. Therefore, this is computed for every possible value $k$ of $\pin$ and set to zero if $k\neq\pin$. The following lemma explains how this is established.

\begin{lemma}\label{Lem:lay2}
	Let $\pin\in\NN$ and $\fin(p)\in\left]0,2\right]$ for every $p\in[p^*]$. Then, for every $p\in[p^*]$ and every $k\in[p-1]$, it holds that $\mathbf{o}_2(p,k)=\fin(p-\pin)$ if and only if $k=\pin$, and $\mathbf{o}_2(p,k)=0$ otherwise.
\end{lemma}
\begin{proof}{Proof.}
	If $k=\pin$, we obtain from Lemma~\ref{Lem:lay1} that $\mathbf{o}_1^{(+)}(k)+\mathbf{o}_1^{(-)}(k)=0$. Thus, due to nonnegativity of $\fin(p-k)$, we obtain $\mathbf{o}_2(p,k) = \sigma(\fin(p-k)) = \fin(p-k)= \fin(p-\pin)$.
	
	If $k\neq\pin$, we obtain from Lemma~\ref{Lem:lay1} that $\mathbf{o}_1^{(+)}(k)+\mathbf{o}_1^{(-)}(k)\geq 2$. Thus, due to $\fin(p-k)\leq 2$, we obtain $\fin(p-k)-\mathbf{o}_1^{(+)}(k)-\mathbf{o}_1^{(-)}(k)\leq0$, which implies $\mathbf{o}_2(p,k)=0$.\qed
\end{proof}

The purpose of the third layer is to help calculating the final minimum. More precisely, it computes how much $\fout(p)$ should be smaller than $\fin(p)$ in the following way.

\begin{lemma}\label{Lem:lay3}
	Let $\pin\in\NN$, $\mysin\in\left]0,1\right]$, and $\fin(p)\in\left]0,2\right]$ for every $p\in[p^*]$. Then $\mathbf{o}_3(p)=\max\{0,\fin(p)-\mysin\}$ for every $p\in[\pin]$ and $\mathbf{o}_3(p)=\max\{0,\fin(p)-(\fin(p-\pin)+\mysin)\}$ for every $p\in\{\pin+1,\pin+2,\dots,p^*\}$.
\end{lemma}
\begin{proof}{Proof.}
	If $p\leq\pin$, we obtain from Lemma~\ref{Lem:lay2} that $\sum_{k=1}^{p-1}\mathbf{o}_2(p,k)=0$. If $p>\pin$, Lemma~\ref{Lem:lay2} implies $\sum_{k=1}^{p-1}\mathbf{o}_2(p,k)=\fin(p-\pin)$. Thus, the claim follows by construction of the third layer and definition of $\sigma$.\qed
\end{proof}

Now, after we have investigated the functionality of each of the hidden layers, we are able to prove the main theorem of this section.

\begin{proof}{Proof of Theorem~\ref{Thm:DPCorrectness}.}
	Using Lemma~\ref{Lem:lay3}, we obtain $\fout(p)=\min\{\fin(p), \mysin\}$ if~$p\leq\pin$ and~$\fout(p)=\min\{\fin(p), \fin(p-\pin)+\mysin\}$ if $p>\pin$. The claim follows due to the Recursion~\eqref{Equ:RecursionTilde} with the respective starting values.\qed
\end{proof}

\section{Smaller RNNs with Provable Approximation Guarantees}\label{Sec:FPTASNN}

In order to overcome the drawback due to the dependence of the network width on~$p^*$, we provide a construction, called
\emph{FPTAS-NN}, that uses less neurons, at the cost of losing optimality. Instead, we prove an approximation ratio (i.e., a worst-case bound) for
the solution value computed by the FPTAS-NN. As in the standard Knapsack FPTAS
\citep{Hochbaum:ApproxAlgs,Vazirani:ApproxAlgs,williamson2011design}, the idea of this construction is to round the
profit values if~$p^*$ becomes too large for an exact computation. Our approximation result can be interpreted as a
tradeoff between the width of the NN and the quality of the Knapsack solution obtained.
As in the previous section, we first give the high-level idea before providing details.

\subsection*{High-level idea of the construction}

Let $P\in\NN$ be a fixed number. The FPTAS-NN computes values $g(p,i)$ for every $p\in [P]$ and $i\in[n]$. These values are similar to the values $f(p,i)$ of the previous section, there is, however, one major difference.
Let $p^*_i=\sum_{j=1}^i p_j$ be the total profit of the first $i$ items. As soon as $p^*_i$ exceeds $P$, we can no longer store a required size value for every possible profit value but have to round profits instead. The granularity we want to use for rounding is $d_i\coloneqq\max\{1,p^*_i/P\}$. We construct the \mbox{FPTAS-NN} to compute values $g(p,i)$, $p\in[P]$, $i\in[n]$, such that we can guarantee the existence of a subset of~$[i]$ that has size at most $g(p,i)$ and profit at least $p\,d_i$. Moreover, this is done in such a way that the optimal solution cannot have a considerably higher profit value. That is, we prove a worst-case approximation guarantee for the solution found by the FPTAS-NN.

In addition to the values of $g$, the FPTAS-NN must also propagate the current total profit value~$p^*_i$ through the network in order to determine the rounding granularity in each step. Hence, in the $i$-th step, it receives $P+3$ inputs, namely $g(p,i-1)$ for $p\in[P]$, $p^*_{i-1}$, $p_i$, and $s_i$. It computes~$P+1$ outputs, namely $g(p,i)$ for $p\in[P]$ and $p^*_{i}$. Figure~\ref{Fig:RNNStructureII} illustrates the recurrent structure of this NN.

\begin{figure}[ht]
	\centering
	\begin{tikzpicture}[scale=\tikzscalefactor, every node/.style={transform shape}]\tikzsize
		\node[draw, ellipse, rotate=90, minimum width = 16ex] (minus1) at (0,0) {$g(\cdot,i-1)$};
		\node[draw, ellipse, rotate=90, minimum width = 16ex] (null) at (36ex,0) {$g(\cdot,i)$};
		\draw [] (9ex, -15ex) rectangle (27ex, 8ex)  node [pos=0.5]{\larger FPTAS-NN};
		\draw [] (45ex, -15ex) rectangle (63ex, 8ex)  node [pos=0.5]{\larger FPTAS-NN};
		\node[bigneuron] (pstarminus1) at (0,-12ex) {$p^*_{i-1}$};
		\node[bigneuron] (pstarnull) at (36ex,-12ex) {$p^*_{i}$};
		\node[bigneuron] (p0) at (13ex,-21ex) {$p_{i}$};
		\node[bigneuron] (s0) at (23ex,-21ex) {$s_{i}$};
		\node[bigneuron] (p1) at (49ex,-21ex) {$p_{i+1}$};
		\node[bigneuron] (s1) at (59ex,-21ex) {$s_{i+1}$};
		\draw[connection] (minus1) -- (9ex,0);
		\draw[connection] (null) -- (45ex,0);
		\draw[connection] (-6ex,0) -- (minus1);
		\draw[connection] (27ex,0) -- (null);
		\draw[connection] (63ex,0) -- (66ex,0);
		\draw[connection] (pstarminus1) -- (9ex,-12ex);
		\draw[connection] (pstarnull) -- (45ex,-12ex);
		\draw[connection] (-6ex,-12ex) -- (pstarminus1);
		\draw[connection] (27ex,-12ex) -- (pstarnull);
		\draw[connection] (63ex,-12ex) -- (66ex,-12ex);
		\draw[connection] (p0) -- (13ex,-15ex);
		\draw[connection] (s0) -- (23ex,-15ex);
		\draw[connection] (p1) -- (49ex,-15ex);
		\draw[connection] (s1) -- (59ex,-15ex);
	\end{tikzpicture}
	\caption{Recurrent structure of the FPTAS-NN for the Knapsack Problem.}
	\label{Fig:RNNStructureII}
\end{figure}
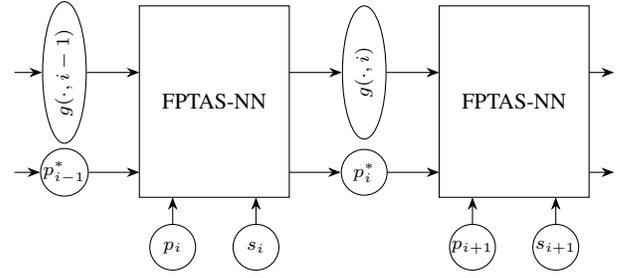

As in Section~\ref{Sec:DPNN}, we use bold symbols in order to distinguish input, activation, and output values that depend on the concrete Knapsack instance from fixed parameters of the network. We again drop the index $i$ in order to make the recurrent structure obvious. We denote the $n_0=P+3$ input parameters by $\gin(p)$, for $p\in[P]$, as well as $\pstarin$, $\pin$, and~$\mysin$. The $P+1$ output values are denoted by $\gout(p)$, for $p\in[P]$, and $\pstarout$.
Similar to the DP-NN in Section~\ref{Sec:DPNN}, the basic idea is to implement a recursion of the type 
\[
\gout(p) = \min\bigl\{\gin(p^{(1)}), \gin(p^{(2)})+\mysin\bigr\}\quad\text{for $p\in[P]$,}
\]
where the first argument of the minimum represents the option of not using item $i$, while the second one corresponds to using it. Notice, however, that $p^{(1)}$ and $p^{(2)}$ cannot simply be calculated as~$p$ and~\mbox{$p-\pin$}, respectively, since we may have to round with different granularities in two successive steps.
Therefore, the rough structure of the FPTAS-NN is as follows: first, $\pstarin$ and $\pin$ are used in order to calculate the old and new rounding granularities~$\dold=\max\{1,\pstarin/P\}$, as well as,~\mbox{$\dnew=\max\{1,(\pstarin+\pin)/P\}$}. Since this computation consists of maxima and weighted sums only, it can easily be achieved by an NN with one hidden layer. Second, the granularities are used in order to select $\gin(p^{(1)})$ and $\gin(p^{(2)})$ from the inputs. Below we give some more details on how this is done. The value of $p^{(2)}$ also depends on $\pin$. Third, the final recursion is established as in the DP-NN. In addition to $\gout(p)$, for $p\in[P]$, we also output $\pstarout=\pstarin+\pin$ in order to keep track of the rounding granularities in subsequent steps. An overview of the entire network structure is given in Figure~\ref{Fig:HighLevelNNII}.

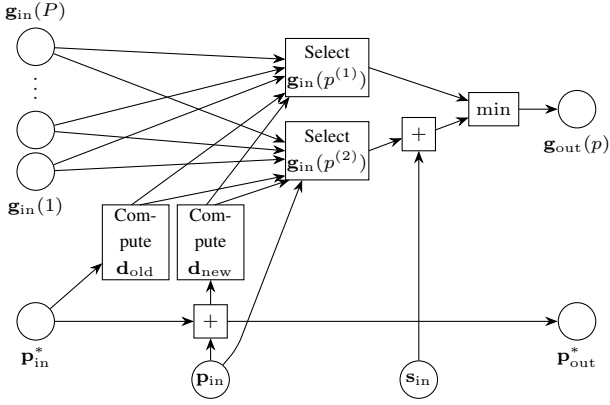
\begin{figure}[ht]
	\centering
	\begin{tikzpicture}[scale=\tikzscalefactor, every node/.style={transform shape}]\tikzsize
		
		\renewcommand{\baselinestretch}{1} 
		
		\node[smallneuron, label=below:{$\gout(p)$}] (out) at (41ex,-2.5ex) {};
		\node[smallneuron, label=below:{$\pstarout$}] (pstarout) at (41ex,-28ex) {};
		
		\node[smallneuron] (si) at (22ex, -35ex) {$\mysin$};
		\node[smallneuron] (pi) at (-3ex, -35ex) {$\pin$};
		
		\node[smallneuron, label=above:{$\gin(P)$}] (inP) at (-24ex,5ex) {};
		\node[rotate=90] at (-24ex,0ex) {$\mathbf{\cdots}$};
		\node[smallneuron] (in2) at (-24ex,-5ex) {};
		\node[smallneuron, label=below:{$\gin(1)$}] (in1) at (-24ex,-10ex) {};
		
		\node[smallneuron, label=below:{$\pstarin$}] (pstarin) at (-24ex,-28ex) {};
		
		\draw (-16ex,-23ex) rectangle (-8ex,-14ex) node [pos=0.5, align=center]{Com-\\pute\\$\dold$};
		\draw (-7ex,-23ex) rectangle (1ex,-14ex) node [pos=0.5, align=center]{Com-\\pute\\$\dnew$};
		\draw (-5ex,-30ex) rectangle (-1ex,-26ex) node [pos=0.5]{$+$};
		
		\draw (6ex,-11ex) rectangle (16ex,-4ex) node [pos=0.5, align=center]{Select\\$\gin(p^{(2)})$};
		\draw (6ex,-1ex) rectangle (16ex,6ex) node [pos=0.5, align=center]{Select\\$\gin(p^{(1)})$};
		
		\draw (28ex,-4.5ex) rectangle (34ex,-0.5ex) node [pos=0.5]{$\min$};
		\draw (20ex,-7.5ex) rectangle (24ex,-3.5ex) node [pos=0.5]{$+$};
		
		\draw[connection] (inP) -- (6ex, 3.5ex);
		\draw[connection] (in2) -- (6ex, 2.5ex);
		\draw[connection] (in1) -- (6ex, 1.5ex);
		\draw[connection] (inP) -- (6ex, -6.5ex);
		\draw[connection] (in2) -- (6ex, -7.5ex);
		\draw[connection] (in1) -- (6ex, -8.5ex);
		
		\draw[connection] (-12.5ex, -14ex) -- (6ex, -0.5ex);
		\draw[connection] (-11.5ex, -14ex) -- (6ex, -10.5ex);
		\draw[connection] (-3.5ex, -14ex) -- (6.5ex, -1ex);
		\draw[connection] (-2.5ex, -14ex) -- (6.5ex, -11ex);
		
		\draw[connection, rounded corners] (pi) -- (0.2ex, -31.2ex) -- (8ex, -11ex);
		\draw[connection] (pstarin) -- (-5ex, -28ex);
		\draw[connection] (pi) -- (-3ex, -30ex);
		\draw[connection] (pstarin) -- (-16ex, -21ex);
		\draw[connection] (-3ex, -26ex) -- (-3ex, -23ex);
		\draw[connection] (-1ex, -28ex) -- (pstarout);
		\draw[connection] (si) -- (22ex, -7.5ex);
		
		\draw[connection] (16ex, -7.5ex) -- (20ex, -6ex);
		\draw[connection] (24ex, -5ex) -- (28ex, -3.5ex);
		\draw[connection] (16ex, 2.5ex) -- (28ex, -1.5ex);
		\draw[connection] (34ex, -2.5ex) -- (out);
		
	\end{tikzpicture}
	\caption{High-level idea how the FPTAS-NN computes $\gout(p)$, $p\in[P]$, and $\pstarout$ from the inputs.}
	\label{Fig:HighLevelNNII}
\end{figure}

Suppose we use the network for processing the $i$-th item. For each $p\in[P]$ we want to determine a (preferably small) value $\gout(p)$ such that there is a subset of~$[i]$ of total profit at least $p\,\dnew$ and total size at most $\gout(p)$. For each $p'\in[P]$, we know that there is a subset of~\mbox{$[i-1]$} of total profit at least $p'\dold$ and total size at most $\gin(p')$. We have two options: ignoring item~$i$ or using it. If we ignore it, then each $p^{(1)}$ with \mbox{$p^{(1)}\dold\geq p\,\dnew$} allows us to choose~\mbox{$\gout(p)=\gin(p^{(1)})$}. If we do use the $i$-th item, however, then each $p^{(2)}$ with the property~\mbox{$p^{(2)}\dold + \pin \geq p\,\dnew$} allows us to choose~\mbox{$\gout(p)=\gin(p^{(2)}) + \mysin$}. Hence, we want to choose $p^{(1)}$ and $p^{(2)}$ as small as possible such that these properties are fulfilled. Therefore, the units labeled \mbox{`Select~$\gin(p^{(1)})$'} and \mbox{`Select~$\gin(p^{(2)})$'} in Figure~\ref{Fig:HighLevelNNII} are constructed by setting all other connections to zero except for those belonging to the smallest values of $p^{(1)}$ and $p^{(2)}$ satisfying the above properties. Similar to how we computed~\mbox{$\fin(p-\pin)$} in the previous section, this requires two hidden layers and $\mathcal{O}(P^2)$ neurons in total.

In total, the FPTAS-NN has depth 5. The first hidden layer computes the rounding granularities, two hidden layers are required to select $\gin(p^{(1)})$ and $\gin(p^{(2)})$ and a final hidden layer computes the minimum in the actual recursion. The width and size of the FPTAS-NN are in the order of $\mathcal{O}(P^2)$.
Unfolding the RNN and viewing it as a single feedforward NN executing the whole FPTAS results in depth $\mathcal{O}(n)$ and size $\mathcal{O}(nP^2)$.

In the next subsection, we provide a formal description of the FPTAS-NN as well as proofs of the following two theorems. The first one ensures that the FPTAS-NN produces only feasible Knapsack solutions, while the second one shows that the FPTAS-NN indeed provides a fully polynomial-time approximation scheme to solve the Knapsack Problem.

\begin{theorem}\label{Thm:Feasibility}
	Suppose the FPTAS-NN is applied to a Knapsack instance with capacity~$S=1$, $s_i\in\left]0,1\right]$, and~$p_i\in\NN$, for $i\in[n]$. For every $i\in[n]$ and every $p\in[P]$ with $g(p,i)\leq 1$, there exists a subset of $[i]$ with profit at least $pd_i$ and size at most $g(p,i)$.
\end{theorem}

\begin{theorem}\label{Thm:Approx}
	For a Knapsack instance with capacity~$S=1$, $s_i\in\left]0,1\right]$, $p_i\in\NN$, for~\mbox{$i\in[n]$}, and for $\varepsilon\in\left]0,1\right]$, set $P\coloneqq\lceil n^2/\varepsilon\rceil$. Let $p^\mathrm{OPT}$ be the profit of the optimal solution and \mbox{$p^\mathrm{NN}=\max\{pd_n\mid g(p,n)\leq1\}$} be the best possible profit found by the FPTAS-NN. Then $p^\mathrm{NN}\geq(1-\varepsilon)p^\mathrm{OPT}$.
\end{theorem}

Theorem~\ref{Thm:Approx} implies a tradeoff between the width of the NN and the precision of the Knapsack solution in the following sense. For achieving an approximation ratio of $1-\varepsilon$, an NN of width $\mathcal{O}(P^2)=\mathcal{O}(n^4/\varepsilon^2)$ is required. In other words, the FPTAS-NN with fixed width $w$ achieves a worst-case approximation ratio of $1-\mathcal{O}(n^2/\sqrt{w})$.

As discussed in the introduction, assuming $\text{P}\neq \text{NP}$, it is clear that the size of a uniform~NN must grow if $\varepsilon$ tends to zero. Hence, complexity theory implies that a width-quality trade-off cannot be avoided. Still, it remains as an open question whether the growth rates implied by our construction are best possible.

\subsection*{Details of the construction and correctness}

We now formally describe the FPTAS-NN and prove that it provides provable approximation guarantees.

As for the DP-NN, we truncate the values of $g$ at 2, that is, instead of any value larger than 2 including~$+\infty$, we just use the value 2.
The FPTAS-NN is applied to a Knapsack instance in the following way. Using the initialization~$g(p,0)=2$ for $p\in[P]$ and $p^*_0=0$, for $i=1,\dots,n$, we feed the inputs~\mbox{$\gin(p)=g(p,i-1)$} for $p\in[P]$, $\pstarin=p^*_{i-1}$, $\pin=p_i$, and $\mysin=s_i$ into the network and store the outputs as~\mbox{$g(p,i)=\gout(p)$} for $p\in[P]$ and $p^*_i=\pstarout$.

\begin{subequations}
	The FPTAS-NN has four hidden layers. After the $n_0=P+3$ input neurons $\gin(p)$ for $p\in[P]$, $\pstarin$, $\pin$, and $\mysin$, the first hidden layer consists of $n_1=2$ neurons $\mathbf{o}_1^\mathrm{old}$ and $\mathbf{o}_1^\mathrm{new}$ which help to compute the rounding granularities $\dold$ and $\dnew$. They are defined as follows:
	
	\noindent\begin{align}
		\mathbf{o}_1^\mathrm{old} &= \sigma\left(\frac{\pstarin}{P}-1\right),\label{Equ:FPTAS-NN-Start}\\
		\mathbf{o}_1^\mathrm{new} &= \sigma\left(\frac{\pstarin+\pin}{P}-1\right),\\
		\dold &= \mathbf{o}_1^\mathrm{old} +1,\\
		\dnew &= \mathbf{o}_1^\mathrm{new} +1.
	\end{align}
	
	The granularities $\dold$ and $\dnew$ are just affine transformations of $\mathbf{o}_1^\mathrm{old}$ and $\mathbf{o}_1^\mathrm{new}$. Hence, they do not form an own hidden layer, because we do not apply the ReLU activation function there. The correct functionality of the first layer is ensured by the following lemma.
	
	\begin{lemma}
		The first layer of the FPTAS-NN correctly computes the rounding granularities $\dold=\max\{1,\frac{\pstarin}{P}\}$ and $\dnew=\max\{1,\frac{\pstarin+\pin}{P}\}$.
	\end{lemma}
	\begin{proof}{Proof.}
		This follows from the fact that $\sigma(x-1)+1=\max\{0,x-1\}+1=\max\{1,x\}$, where $x$ equals either $\frac{\pstarin}{P}$ or $\frac{\pstarin+\pin}{P}$.\qed
	\end{proof}
	
	Hence, in the $i$-th step, if we feed the inputs $\pstarin=p^*_{i-1}$ and $\pin=p_i$ into the network, $\dold$ and $\dnew$ equal $d_{i-1}$ and $d_i$, respectively.
	
	In the second hidden layer, we have a total of $n_2=2 P^2 + 2 P$ hidden neurons, denoted by $\mathbf{o}_2^{(1+)}(p,k)$ and $\mathbf{o}_2^{(1-)}(p,k)$ for $p,k\in[P]$ with $p\leq k$, as well as, $\mathbf{o}_2^{(2+)}(p,k)$ and $\mathbf{o}_2^{(2-)}(p,k)$ for $p,k\in[P]$ with $p\geq k$, defined as follows:
	\begin{align}
		\mathbf{o}_2^{(1+)}(p,k) &= \sigma(2P(p\dnew-k\dold)),&&p,k\in[P], p\leq k,\\
		\mathbf{o}_2^{(1-)}(p,k) &= \sigma(2P((k-1)\dold-p\dnew)+2),&&p,k\in[P], p\leq k,\\
		\mathbf{o}_2^{(2+)}(p,k) &= \sigma(2P(p\dnew-k\dold-\pin)),&&p,k\in[P], p\geq k,\\
		\mathbf{o}_2^{(2-)}(p,k) &= \sigma(2P((k-1)\dold+\pin-p\dnew)+2),&&p,k\in[P], p\geq k.
	\end{align}
	
	For a fixed $p\in[P]$, let $p^{(1)}$ and $p^{(2)}$ be the smallest possible integers with $p^{(1)}\dold \geq p\dnew$ and $p^{(2)}\dold + \pin \geq p\dnew$, respectively. The task of the second layer is to detect the values $p^{(1)}$ and $p^{(2)}$, as formalized by the following two lemmas.
	
	\begin{lemma}\label{Lem:lay21}
		For each $p,k\in[P]$ with $p\leq k$, we have $\mathbf{o}_2^{(1+)}(p,k)+\mathbf{o}_2^{(1-)}(p,k)=0$ if and only if $k=p^{(1)}$. Otherwise, we have $\mathbf{o}_2^{(1+)}(p,k)+\mathbf{o}_2^{(1-)}(p,k)\geq2$.
	\end{lemma}
	
	\begin{proof}{Proof.}
		Obviously, it holds that $\mathbf{o}_2^{(1+)}(p,k)=0$ if and only if $k\dold \geq p\dnew$. On the other hand, using that $\dold$ and $\dnew$ are integer multiples of $\frac{1}{P}$, we obtain
		
		\noindent
		\begin{align*}
			& \mathbf{o}_2^{(1-)}(p,k)=0\\
			\Leftrightarrow\quad & (k-1)\dold \leq p\dnew - \frac{1}{P}\\
			\Leftrightarrow\quad & (k-1)\dold < p\dnew\\
			\Leftrightarrow\quad & \text{no integer $k'<k$ satisfies $k'\dold \geq p\dnew$.}
		\end{align*}
		
		This proves the first part of the claim. The second part follows because, again, $\dold$ and $\dnew$ are integer multiples of $\frac{1}{P}$ and, hence, $\mathbf{o}_2^{(1+)}(p,k)+\mathbf{o}_2^{(1-)}(p,k)$ is an integer multiple of~$2$.\qed
	\end{proof}
	
	\begin{lemma}\label{Lem:lay22}
		For each $p,k\in[P]$ with $p\geq k$, we have $\mathbf{o}_2^{(2+)}(p,k)+\mathbf{o}_2^{(2-)}(p,k)=0$ if and only if $k=p^{(2)}$. Otherwise, we have $\mathbf{o}_2^{(2+)}(p,k)+\mathbf{o}_2^{(2-)}(p,k)\geq2$.
	\end{lemma}
	\begin{proof}{Proof.}
		Analogous to Lemma~\ref{Lem:lay21}.\qed
	\end{proof}
	
	The third hidden layer consists of $n_3=P^2 + P$ neurons $\mathbf{o}_3^{(1)}(p,k)$ for $p,k\in[P]$ with $p\leq k$, as well as $\mathbf{o}_3^{(2)}(p,k)$ for $p,k\in[P]$ with $p\geq k$. Moreover, we have again helping variables that do not form an own hidden layer because they are only affine transformations of the previous layers, namely $\mathbf{h}^{(1)}(p)$ and $\mathbf{h}^{(2)}(p)$ for $p\in[P]$.
	\begin{align}
		\mathbf{o}_3^{(1)}(p,k) &= \sigma(2-\gin(k)-\mathbf{o}_2^{(1+)}(p,k)-\mathbf{o}_2^{(1-)}(p,k)),&&p,k\in[P], p\leq k,\\
		\mathbf{o}_3^{(2)}(p,k) &= \sigma(\gin(k)-\mathbf{o}_2^{(2+)}(p,k)-\mathbf{o}_2^{(2-)}(p,k)),&&p,k\in[P], p\geq k,\\
		\mathbf{h}^{(1)}(p) &= 2-\sum_{k=p}^{P} \mathbf{o}_3^{(1)}(p,k), &&p\in[P],\\
		\mathbf{h}^{(2)}(p) &= \sum_{k=1}^{p} \mathbf{o}_3^{(2)}(p,k), &&p\in[P].
	\end{align}
	
	The idea of this layer is to compute $\gin(p^{(1)})$ and $\gin(p^{(2)})$, as the following two lemmas show.
	
	\begin{lemma}\label{Lem:lay31}
		For each $p\in[P]$, if $p^{(1)}\leq P$, we have $\mathbf{h}^{(1)}(p)=\gin(p^{(1)})$. If $p^{(1)}>P$, we have $\mathbf{h}^{(1)}(p)=2$.
	\end{lemma}
	\begin{proof}{Proof.}
		Note that $p^{(1)}$ is never smaller than $p$. If $p\leq p^{(1)}\leq P$, then $\mathbf{o}_3^{(1)}(p,p^{(1)})=2-\gin(p^{(1)})$ and $\mathbf{o}_3^{(1)}(p,k)=0$ for each $k\neq p^{(1)}$ by Lemma~\ref{Lem:lay21}. If $p^{(1)}>P$, then $\mathbf{o}_3^{(1)}(p,k)=0$ for each $k$. Thus, the claim follows by the definition of $\mathbf{h}^{(1)}$.\qed
	\end{proof}
	
	\begin{lemma}\label{Lem:lay32}
		For each $p\in[P]$, if $p^{(2)}\geq 1$, we have $\mathbf{h}^{(2)}(p)=\gin(p^{(2)})$. If $p^{(2)}\leq0$, we have $\mathbf{h}^{(2)}(p)=0$.
	\end{lemma}
	\begin{proof}{Proof.}
		We first show that $p^{(2)}$ is never larger than $p$ by proving that $p\dold + \pin \geq p\dnew$. If $\dnew=1$, then also $\dold=1$ holds and this statement is true. Otherwise, we have $\dnew = \frac{\pstarin+\pin}{P}$ and $\dold\geq\frac{\pstarin}{P}$. Hence, we obtain $p(\dnew-\dold) \leq \frac{p\pin}{P} \leq \pin$. Therefore, in any case, $p\dold + \pin \geq p\dnew$ follows, and thus also $p^{(2)}\leq p$.
		
		If $1\leq p^{(2)}\leq p$, then it follows that $\mathbf{o}_3^{(2)}(p,p^{(2)})=\gin(p^{(2)})$ and $\mathbf{o}_3^{(2)}(p,k)=0$ for each $k\neq p^{(2)}$ by Lemma~\ref{Lem:lay22}. If $p^{(2)}\leq 0$, then $\mathbf{o}_3^{(2)}(p,k)=0$ holds for each $k$. Thus, the claim follows by the definition of $\mathbf{h}^{(2)}$.\qed
	\end{proof}
	
	The fourth hidden layer is used to compute the minimum in the recursion and consists of $n_4=P$ neurons $\mathbf{o}_4(p)$ for $p\in[P]$. Finally, we output the $P$ values $\gout(p)$ for $p\in[P]$, as well as $\pstarout=\pstarin+\pin$.
	
	\noindent\begin{align}
		\mathbf{o}_4(p) &= \sigma(\mathbf{h}^{(1)}(p)-(\mysin+\mathbf{h}^{(2)}(p))), &&p\in[P],\\
		\gout(p) &= \mathbf{h}^{(1)}(p)- \mathbf{o}_4(p), &&p\in[P],\\
		\pstarout&=\pstarin+\pin.\label{Equ:FPTAS-NN-Final}
	\end{align}
	The following lemma ensures that the output value $\gout(p)$ is indeed computed by the desired recursion, provided that $\mathbf{h}^{(1)}$ and $\mathbf{h}^{(2)}$ are computed properly.
	
	\begin{lemma}\label{Lem:lay4}
		For each $p\in[P]$, we have $\gout(p)=\min\{\mathbf{h}^{(1)}(p),\mysin+\mathbf{h}^{(2)}(p)\}$.
	\end{lemma}
	\begin{proof}{Proof.}
		This final layer of the FPTAS-NN is constructed exactly in the same way as the~NN in Figure~\ref{Fig:Min2Num}.\qed
	\end{proof}
\end{subequations}

Equations~\eqref{Equ:FPTAS-NN-Start} to~\eqref{Equ:FPTAS-NN-Final} fully define the FPTAS-NN. We have shown several lemmas concerning the functionality of the individual layers. Now we turn towards the proofs of Theorems~\ref{Thm:Feasibility} and~\ref{Thm:Approx}.

\begin{proof}{Proof of Theorem~\ref{Thm:Feasibility}.}
	We show that the claim even holds for all values of $p$ and $i$ with $g(p,i)<2$ and not only for those with $g(p,i)\leq 1$.
	
	We use induction on $i$. For the induction start ($i=0$), nothing is to show due to the initialization $g(p,0)=2$ for all $p\in[P]$. For the induction step, suppose the claim is valid for all steps up to $i-1$.
	
	Fix some $p\in[P]$. By Lemma~\ref{Lem:lay4}, the output $g(p,i)=\gout(p)$ in the $i$-th step equals $\min\{\mathbf{h}^{(1)}(p),\mysin+\mathbf{h}^{(2)}(p)\}$. In the following, we distinguish two cases. Recall that $p^{(1)}$ and $p^{(2)}$ are the smallest possible integers with $p^{(1)}\dold \geq p\dnew$ and $p^{(2)}\dold + \pin \geq p\dnew$, respectively.
	
	\emph{Case 1:} $\mathbf{h}^{(1)}(p)\leq\mysin+\mathbf{h}^{(2)}(p)$. This implies $g(p,i)=\mathbf{h}^{(1)}(p)$. If $\mathbf{h}^{(1)}(p)=2$, nothing is to show. Otherwise, by Lemma~\ref{Lem:lay31}, we have $p^{(1)}\leq P$ with $p^{(1)}\dold\geq p\dnew$ and $g(p,i)=\mathbf{h}^{(1)}(p)=\gin(p^{(1)})=g(p^{(1)},i-1)$. By induction, we obtain that there exists a subset of $[i-1]$ with size at most $g(p,i)$ and profit at least $p^{(1)}\dold$. Hence, using the same items yields a subset of $[i]$ with size at most $g(p,i)$ and profit at least $p\dnew$. Thus, the claim is proven in this case.
	
	\emph{Case 2:} $\mathbf{h}^{(1)}(p)>\mysin+\mathbf{h}^{(2)}(p)$. This implies $g(p,i)=\mysin+\mathbf{h}^{(2)}(p)$. Note that this can only happen if $\mathbf{h}^{(2)}(p)<2$ because $\mathbf{h}^{(1)}(p)$ has at most value 2. First, suppose $p^{(2)}\leq 0$. This implies $p_i=\pin\geq p\dnew$. Hence, by using just item $i$, we obtain a subset of profit at least $p\dnew$ and size at most $s_i=\mysin\leq \mysin+\mathbf{h}^{(2)}(p) = g(p,i)$, and we are done. Second, if $p^{(2)}\geq 1$, then Lemma~\ref{Lem:lay32} implies that $g(p,i)=\mysin+\mathbf{h}^{(2)}(p)=\mysin+\gin(p^{(2)})=s_i+g(p^{(2)},i-1)$. By induction, we obtain that there exists a subset of $[i-1]$ with size at most $g(p,i)-s_i$ and profit at least $p^{(2)}\dold$. Hence, adding item $i$ to this subset yields a subset of $[i]$ with size at most $g(p,i)$ and profit at least $p^{(2)}\dold+p_i\geq p\dnew$. Thus, the claim is also proven in this case.\qed
\end{proof}

\begin{proof}{Proof of Theorem~\ref{Thm:Approx}.}
	Let $M^\mathrm{OPT}$ be an optimal solution to the Knapsack instance and $M^\mathrm{OPT}_i=M^\mathrm{OPT}\cap[i]$ be the subset of $[i]$ chosen by the optimal solution. Let $s^\mathrm{OPT}_i=\sum_{j\in M^\mathrm{OPT}_i} s_j$ be the size of $M^\mathrm{OPT}_i$ and $p^\mathrm{OPT}_i=\sum_{j\in M^\mathrm{OPT}_i} p_j$ be the profit of $M^\mathrm{OPT}_i$. The idea of the proof is that in each step, we lose at most a profit of $d_i$ compared to the optimal solution. Formally, we prove the following claim by induction on $i$: for every $i\in[n]$, and every $p\leq \left\lceil\frac{p^\mathrm{OPT}_i}{d_i}\right\rceil-i$ we have $g(p,i)\leq s^\mathrm{OPT}_i$.
	
	The induction start is settled by extending the claim to $i=0$, for which it is trivial. For the induction step, suppose the claim is valid for all steps up to $i-1$. Fix a value $p\leq \left\lceil\frac{p^\mathrm{OPT}_i}{d_i}\right\rceil-i$. Let again $p^{(1)}$ and $p^{(2)}$ be the smallest possible integers with $p^{(1)}d_{i-1} \geq pd_i$ and $p^{(2)}d_{i-1} + p_i \geq pd_i$, respectively. We distinguish two cases.
	
	\emph{Case 1:} $i\notin M^\mathrm{OPT}$, i.e., the optimal solution does not use item $i$. Observe that
	
	\noindent\begin{align*}
		pd_i&\leq \left(\left\lceil\frac{p^\mathrm{OPT}_{i}}{d_{i}}\right\rceil-i\right)d_{i}\\
		&\leq p^\mathrm{OPT}_i-(i-1)d_i\\
		&= p^\mathrm{OPT}_{i-1}-(i-1)d_i\\
		&\leq p^\mathrm{OPT}_{i-1}-(i-1)d_{i-1}\\
		&\leq \left(\left\lceil\frac{p^\mathrm{OPT}_{i-1}}{d_{i-1}}\right\rceil-(i-1)\right)d_{i-1}.
	\end{align*}
	Hence, we obtain
	\begin{equation}\label{Equ:p1bound}
		p^{(1)}\leq\left\lceil\frac{p^\mathrm{OPT}_{i-1}}{d_{i-1}}\right\rceil-(i-1)
	\end{equation}
	by the definition of $p^{(1)}$. In particular, $p^{(1)}\leq \frac{p^*_{i-1}}{d_{i-1}}\leq P$ by the definition of $d_{i-1}$. Therefore, Lemmas~\ref{Lem:lay31} and~\ref{Lem:lay4} imply $g(p,i)\leq g(p^{(1)},i-1)$. Due to Inequality~\eqref{Equ:p1bound}, it further follows by induction that $g(p,i)\leq g(p^{(1)},i-1)\leq s^\mathrm{OPT}_{i-1} = s^\mathrm{OPT}_i$, which settles the induction step in this case.
	
	\emph{Case 2:} $i\in M^\mathrm{OPT}$, i.e., the optimal solution uses item $i$. Observe that
	
	\noindent\begin{align*}
		pd_i&\leq \left(\left\lceil\frac{p^\mathrm{OPT}_{i}}{d_{i}}\right\rceil-i\right)d_{i}\\
		&\leq p^\mathrm{OPT}_i-(i-1)d_i\\
		&= p^\mathrm{OPT}_{i-1}+p_i-(i-1)d_i\\
		&\leq p^\mathrm{OPT}_{i-1}+p_i-(i-1)d_{i-1}\\
		&\leq \left(\left\lceil\frac{p^\mathrm{OPT}_{i-1}}{d_{i-1}}\right\rceil-(i-1)\right)d_{i-1}+p_i.
	\end{align*}
	
	\noindent Hence, we obtain
	\begin{equation}\label{Equ:p2bound}
		p^{(2)}\leq\left\lceil\frac{p^\mathrm{OPT}_{i-1}}{d_{i-1}}\right\rceil-(i-1)
	\end{equation}
	by the definition of $p^{(2)}$. If $p^{(2)}\leq0$, Lemmas~\ref{Lem:lay32} and~\ref{Lem:lay4} imply $g(p,i)\leq s_i \leq s^\mathrm{OPT}_i$. If $p^{(2)}\geq0$, Lemmas~\ref{Lem:lay32} and~\ref{Lem:lay4} imply $g(p,i)\leq g(p^{(2)},i-1)+s_i$. Due to Inequality~\eqref{Equ:p2bound}, we can further conclude by induction that $g(p,i)\leq g(p^{(2)},i-1)+s_i\leq s^\mathrm{OPT}_{i-1} +s_i= s^\mathrm{OPT}_i$, which finalizes the induction step.
	
	Now, making use of the claim we have just proven by induction, we obtain that~\mbox{$g\left(\left\lceil\frac{p^\mathrm{OPT}}{d_n}\right\rceil-n,n\right)\leq s^\mathrm{OPT}_n \leq 1$}. Therefore, it holds that
	\begin{equation}\label{Equ:pnnbound}
		p^\mathrm{NN}\geq \left(\left\lceil\frac{p^\mathrm{OPT}}{d_n}\right\rceil-n\right)d_n\geq p^\mathrm{OPT}-nd_n.
	\end{equation}
	
	If $d_n=1$, that is, $p^*\leq P$, then we have for all $i\in[n]$ that $d_i=1$. Hence, in each step and for each $p\in[P]$, we have $p^{(1)}=p$ and $p^{(2)}=p-p_i$. Therefore, by Lemmas~\ref{Lem:lay31}--\ref{Lem:lay4}, the FPTAS-NN behaves like the DP-NN from Section~\ref{Sec:DPNN} that executes the exact dynamic program and the theorem follows.
	
	Otherwise, if $d_n>1$, we have $d_n=\frac{p^*}{P}$. Since there must exist one item with profit at least $\frac{p^*}{n}$, we obtain $p^\mathrm{OPT}\geq\frac{p^*}{n}$ and, hence, $nd_n=\frac{np^*}{P}\leq \frac{n^2 p^\mathrm{OPT}}{P}$. Together with~\eqref{Equ:pnnbound}, this implies $\frac{p^\mathrm{NN}}{p^\mathrm{OPT}}\geq 1-\frac{n^2}{P}\geq 1-\varepsilon$.\qed
\end{proof}

\section{Empirical Evidence for Superlinear Width}\label{Sec:Experiments}

While the running time of the classical Knapsack dynamic program depends only linearly on $p^*$, the width of the \mbox{DP-NN} is of the order of $(p^*)^2$. In our construction, the quadratic factor arises from dynamically finding $\fin(p-\pin)$ in a hard-coded network, as explained in Section~\ref{Sec:DPNN}. For similar reasons, the width of the FPTAS-NN grows with $1/\varepsilon^2$ instead of~$1/\varepsilon$.

The natural question to ask is whether one can achieve a linear dependence with a different construction. While this question remains open from a purely theoretical point of view, in this section, we provide some empirical evidence that the required width might indeed grow in a superlinear manner.

We emphasize that the purpose of our empirical study is to validate our theoretical findings on the expressivity of neural networks. We do not propose a method to practically solve knapsack problems.

For details about the experimental setup, including used soft- and hardware, random data generation and systematic of seeds, training and validation setup, hyperparameters, as well as the source code, please refer to Appendix~\ref{Sec:ExpDetails}. Here we only include the necessary information to understand the key findings.

Similar to the DP-NN of Section~\ref{Sec:DPNN}, we train an NN with three hidden layers and variable width to execute one step of the Knapsack dynamic program, that is, to map $\fin$, $\pin$, and $\mysin$ to $\fout$, for random Knapsack instances. For the 30 different values $\{3,6,9,\dots,90\}$ of $p^*$, we increase the width neuron by neuron until a mean squared error (MSE) loss of at most $0.005$ is reached on a validation set. The threshold $0.005$ is carefully chosen such that NNs with reasonable width are empirically able to achieve it. In order to show that our results are robust with respect to the chosen threshold, we repeat the experiment with two different thresholds $0.00375$ and $0.0025$. Since smaller thresholds naturally make the training process more difficult, we use smaller values of $p^*$ in these cases. The empirically required width to achieve an MSE of at most the respective thresholds is plotted in Figures~\ref{Fig:experiment},~\ref{Fig:experiment60}, and~\ref{Fig:experiment30}, respectively.

\newcommand{\figscalefactor}{0.6}

\begin{figure}[t]
	\centering
	\includegraphics[width=\figscalefactor\linewidth]{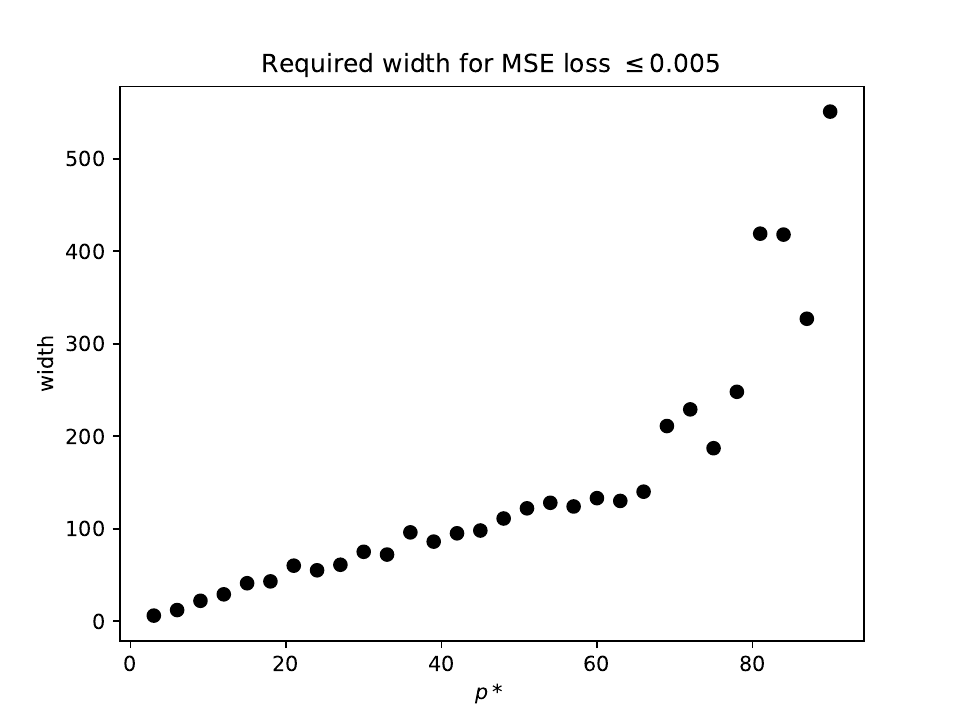}
	\caption{Required width to achieve a mean squared error of at most $0.005$ as a function of $p^*$.}
	\label{Fig:experiment}
\end{figure}

\begin{figure}[t]
	\centering
	\includegraphics[width=\figscalefactor\linewidth]{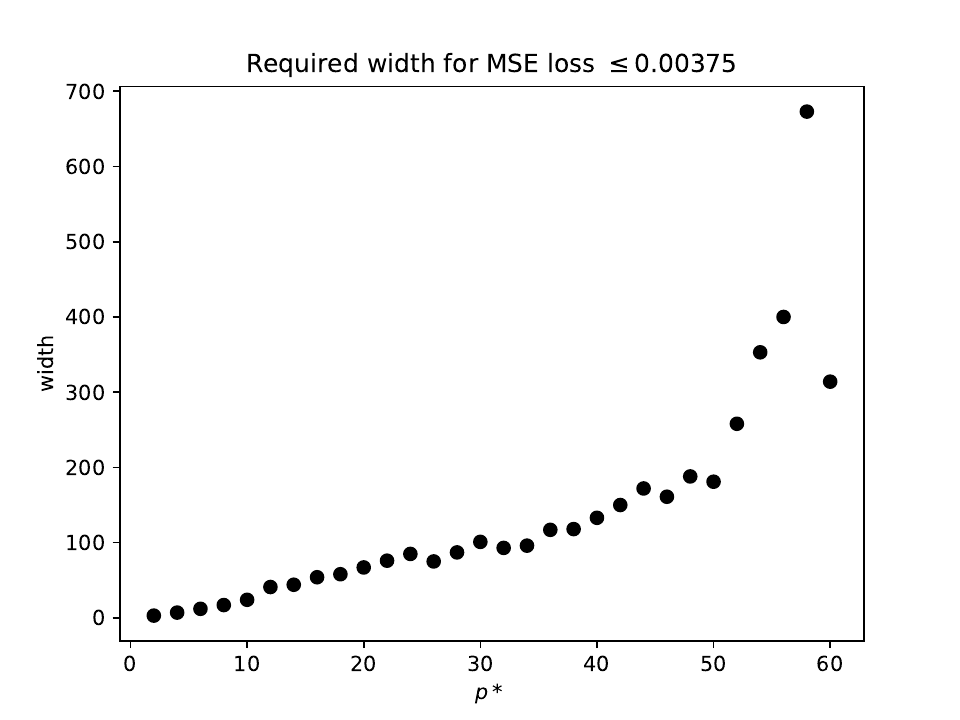}
	\caption{Required width to achieve a mean squared error of at most $0.00375$ as a function of $p^*$.}
	\label{Fig:experiment60}
\end{figure}

\begin{figure}[t]
	\centering
	\includegraphics[width=\figscalefactor\linewidth]{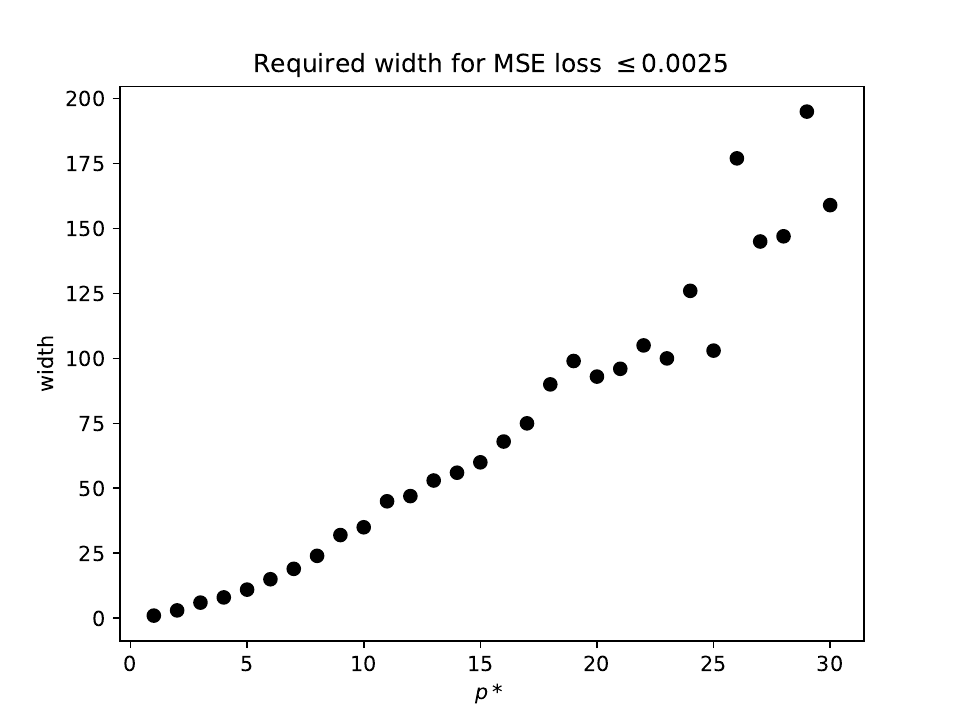}
	\caption{Required width to achieve a mean squared error of at most $0.0025$ as a function of $p^*$.}
	\label{Fig:experiment30}
\end{figure}

As one can see in these figures, the results of the experiments suggest that the dependence of the required width to reach the MSE threshold seems to grow superlinearly with $p^*$. We interpret this as a mild experimental validation of our theoretical construction which requires quadratic instead of linear width.

Of course, one should take these result with a grain of salt since the superlinear relation might have multiple reasons. For instance, it is unclear whether the difficulty to train larger networks has a stronger effect than the expressivity of ReLU NNs. Still, we find that this computational study supports our theoretical size bounds.

\section{Neural Networks for Other CO Problems}\label{Sec:OtherProblems}

In this section we demonstrate that our approach is by no means bound to the Knapsack Problem. In fact, for many other CO problems it is possible to convert a dynamic programming solution method into a provably correct NN. For certain NP-hard CO problems, a dynamic programming solution even implies the existence of a fully polynomial-time approximation scheme \citep{woeginger2000FPTAS}. This, in turn, might shed light on the tradeoff between size of corresponding NNs and their solution quality, as for the Knapsack Problem in Section~\ref{Sec:FPTASNN}. In the following we provide several examples in order to support these claims.

\subsection*{Longest Common Subsequence}
First, consider the problem of finding the length of the longest common subsequence of two finite integer sequences $x_1, \dots, x_m$ and $y_1, \dots, y_n$. A standard dynamic programming procedure, see, e.g., \citep[Section~15.4]{cormen2001introduction}, computes values $f(i,j)$ equal to the length of the longest common subsequence of the partial sequences~\mbox{$x_1, x_2, \dots, x_i$} and \mbox{$y_1, y_2, \dots, y_j$} by applying the recursion
\[
f(i,j)=\left\{
\begin{array}{ll}
	f(i-1,j-1)+1&\text{if }x_i=y_j,\\
	\max\bigl\{f(i-1,j),f(i,j-1)\bigr\}&\text{if }x_i\neq y_j.
\end{array}
\right.
\]
Since the sequence consists of integers, the check whether $x_i$ equals $y_j$ can be performed similarly to checking whether $p'=\pin$ in Section~\ref{Sec:DPNN}. The remainder of the recursion only consists of maxima and sums. Hence, computing $f(i,j)$ from $f(i-1,j-1)$, $f(i-1,j)$, $f(i,j-1)$,~$x_i$, and~$y_j$ can be realized via an NN of constant size. These basic units can be plugged together in a two-dimensional way for computing all values $f(i,j)$, $i\in[m]$, $j\in[n]$. The resulting NN can be seen as a two-dimensional RNN  of constant size that is applied in an $m$ by $n$ grid structure, an architecture introduced by \citet{graves2007multi}. Unfolding the RNN results in a feedforward NN of depth $\mathcal{O}(m+n)$ and size $\mathcal{O}(mn)$ for computing the length of the longest common subsequence.

\subsection*{Single-Source Shortest Path Problem} As a second example, we consider the Bellman-Ford algorithm for the Single-Source Shortest Path Problem, see, e.g., \citep[Section~6.8]{kleinberg2006algorithm}. If $(c_{uv})_{u,v\in V}$ is the length matrix of a graph with vertex set $V$ and $s\in V$ is the source vertex, this algorithm recursively computes values $f(i,v)$ determining the shortest possible length of a path from~$s$ to~$v$ using at most~$i$ arcs by
\[
f(i,v)=\min_{u\in V}\{f(i-1,u)+c_{uv}\}.
\]
Since this recursion consists only of sums and minima, it can be easily implemented in an NN. The sequential time complexity of the Bellman-Ford algorithm on complete digraphs with~\mbox{$n=\abs{V}$} is $\mathcal{O}(n^3)$, which can naturally be parallelized into $\mathcal{O}(n)$ rounds. Since the best known NNs for computing the minimum of $n$ numbers require $\mathcal{O}(\log n)$ depth \citep{Arora:DNNwithReLU}, there exists an NN executing the Bellman-Ford algorithm with depth $\mathcal{O}(n\log n)$ and size $\mathcal{O}(n^3\log n)$. Observe that in each round $i\in[n]$, the computation mapping the values $f(i-1,v)$, $v\in V$, to $f(i,v)$, $v\in V$, is the same. Therefore, this NN can also be seen as an RNN of depth $\mathcal{O}(\log n)$ and size $\mathcal{O}(n^2\log n)$ that is applied $n$ times.

\subsection*{All-Pairs Shortest Path Problem} Third, recall that the All-Pairs Shortest Path Problem can be solved by computing the \mbox{$(n-1)$-th} min-plus matrix power of the length matrix $(c_{uv})_{u,v\in V}$, see, e.g., \citep[Section~2.5.4]{leighton2014introduction}. By repeated squaring, this can be achieved with only $\mathcal{O}(\log n)$ min-plus matrix multiplications. For a single multiplication it is required to compute $\mathcal{O}(n^2)$ times in parallel the minimum of $n$ numbers. One of these minimum computations requires depth $\mathcal{O}(\log n)$ and size $\mathcal{O}(n\log n)$. Putting them in parallel to execute one min-plus matrix product results in depth~$\mathcal{O}(\log n)$ and size $\mathcal{O}(n^3\log n)$. Note that the whole execution consists of $\mathcal{O}(\log n)$ repetitions of the same procedure, namely squaring a matrix in the min-plus sense. Hence, this can again be viewed as an RNN with depth $\mathcal{O}(\log n)$ and size $\mathcal{O}(n^3\log n)$ that is applied $\mathcal{O}(\log n)$ times. Unfolding results in a single feedforward NN with depth $\mathcal{O}(\log^2 n)$ and size $\mathcal{O}(n^3\log^2 n)$ for solving the All-Pairs Shortest Path Problem. 

\subsection*{Constrained Shortest Path Problem} Next, consider a common generalization of the Shortest Path Problem and the Knapsack Problem, namely the NP-hard Constrained Shortest Path Problem. Here, in addition to a (nonnegative) length matrix $(c_{uv})_{u,v\in V}$, the input graph is also equipped with a (nonnegative) resource matrix $(r_{uv})_{u,v\in V}$. The task is to find a minimum length path $P$ from a source vertex $s$ to any other vertex, but this time subject to a resource constraint $\sum_{(u,v)\in P} r_{uv}\leq R$ for a given resource limit $R$. An extensive overview of solution approaches to this problem can be found, e.g., in the dissertation by \citet{ziegelmann2001constrained}. Similar to the Knapsack Problem, there exist two natural pseudo-polynomial dynamic programs, one of them parametrized by length values and the other one by resource values. Both can be implemented on an NN by combining the ideas from Section~\ref{Sec:DPNN} with the NN for the Bellmann-Ford algorithm above. We showcase this for the variant parametrized by the length values. This dynamic program recursively calculates values $f(c,v)$ representing the minimum amount of resource needed for a path from $s$ to $v$ with length at most $c$ by
\[f(c,v)=\min\bigl\{f(c-1,v),\min\nolimits_{u\in V\setminus\{v\}}\{f(c-c_{uv}, u) + r_{uv}\}\bigr\}.\]

For fixed $c$, $u$, and $v$, the term $f(c-c_{uv}, u) + r_{uv}$ can be calculated by a similar construction as we computed $\fin(p-\pin)+\mysin$ in Section~\ref{Sec:DPNN}. Assuming an upper bound~$c^*$ on the optimal objective value, this can be achieved with constant depth and $\mathcal{O}(c^*)$ width. Hence, it remains to compute a minimum of at most $n$ numbers in order to compute $f(c,v)$. Thus, each single value $f(c,v)$ can be computed with depth $\mathcal{O}(\log n)$ and size $\mathcal{O}(nc^*\log n)$. We have to compute $\mathcal{O}(nc^*)$ of these values, but for fixed $c$, all these values can be computed in parallel. Therefore, the whole dynamic program can be executed on an NN with depth $\mathcal{O}(c^* \log n)$ and a total size of $\mathcal{O}(n^2(c^*)^2\log n)$. This is pseudo-polynomial, which is the best we can hope for due to the NP-hardness of the problem, if we aim for uniform NN constructions. Moreover, similar to the Knapsack Problem, this dynamic program can be turned into an FPTAS by downscaling and rounding the length values. This observation can be used to obtain a width-quality tradeoff for the Constrained Shortest Path Problem similar to what we have shown in Section~\ref{Sec:FPTASNN}.

\subsection*{Traveling Salesperson Problem} Finally, consider the Bellman-Held-Karp algorithm for solving the (asymmetric) Traveling Salesperson Problem (TSP); see \citep{bellman1962dynamic, held1962dynamic}. Given a (complete, directed) graph with vertex set $V$ and distances $c_{uv}$ from vertex $u\in V$ to vertex $v\in V$, the TSP asks for the shortest round-trip visiting each vertex exactly once. Choosing an arbitrary starting vertex $s\in V$, the Bellman-Held-Karp algorithm recursively computes values $f(T,v)$ for each $T\subseteq V\setminus\{s\}$, $v\in T$, corresponding to the length of the shortest $s$-$v$-path visiting exactly the nodes in $T\cup\{s\}$ by the formula
\[
f(T,v)=\min_{u\in T\setminus\{v\}}\left\{f(T\setminus\{v\}, u) + c_{uv}\right\}.
\]
The length of the shortest TSP tour is then given by $\min_{u\in V\setminus\{s\}}\left\{f(V\setminus\{s\}, u) + c_{us}\right\}$. While the sequential time complexity of this algorithm on digraphs with $n=\abs{V}$ is $\mathcal{O}(n^22^n)$, in an NN we can compute the values of $f$ for all $T$ with equal cardinality in parallel. As before, computing the minimum introduces an additional factor of $\log n$ in the depth and size of the network. Hence, in total, the TSP can be solved with an NN of depth $\mathcal{O}(n\log n)$ and size $\mathcal{O}(n^22^n\log n)$. In particular, a polynomially deep NN suffices to solve the NP-hard (asymmetric) TSP.

\section{Conclusions and Future Work}\label{Sec:Conclusion}

An obvious open problem is to improve the obtained bounds on the required width of our neural network constructions. In particular, an interesting question is whether meaningful lower bounds beyond those immediately implied by the NP-hardness of the Knapsack Problem can be obtained, as suggested by our experimental results.

Another interesting direction would be to design and conduct a more sophisticated empirical study to analyze the expressivity of neural networks for the Knapsack Problem and related problems. For example, one could aim to find out whether (larger) real-world instances have a certain structure allowing that smaller neural networks capture already the dynamic programming steps for these instances. Note that an exhaustive search for the width as performed in this paper is computationally infeasible for larger instances, so such a study would require a different but still robust experimental setup.

Notice that our networks only output the solution value but not the corresponding solution, i.e., subset of
items. It is easy to see that, as for the dynamic program solving the Knapsack Problem, the subset of items can
be obtained in a straightforward way via backtracking. On the other hand, notice that it is impossible for a
ReLU~NN (without threshold gates) to output (the characteristic vector of) the optimum subset of items: while the function computed by a
ReLU~NN is piecewise linear and continuous \citep{Arora:DNNwithReLU, goodfellow2013maxout}, already infinitesimal changes
of the input (i.e., the profit values of items) might change the optimum subset of items.

Finally, an exciting direction for future research is to generalize our results of Section~\ref{Sec:OtherProblems} by describing a general procedure to convert dynamic programs into ReLU~NNs. Ideally, one could exactly classify the type of dynamic programs that guarantee the existence of a corresponding ReLU~NN. Similar in spirit, \citet{woeginger2000FPTAS} classifies dynamic programs that guarantee the existence of a fully polynomial-time approximation scheme.

\ACKNOWLEDGMENT{%
The authors would like to thank the referees for their insightful comments which helped to improve the presentation of this paper.

A large portion of this research was conducted while the first author was affiliated with TU Berlin and supported by DFG-GRK 2434 ``Facets of Complexity''. Furthermore, this work is supported by the European Research Council (ERC) under the European Union’s Horizon 2020 research and innovation programme (grant agreement ScaleOpt–757481) and by the Deutsche Forschungsgemeinschaft (DFG, German Research Foundation) under Germany's Excellence Strategy --- The Berlin Mathematics Research Center MATH+ (EXC-2046/1, project ID: 390685689).
}

%
\begin{APPENDIX}{}
	\section{Detailed Experimental Setup}
	
	\label{Sec:ExpDetails}
	In this section we describe in detail how we conducted the experiments of Section~\ref{Sec:Experiments} in the main paper.
	
	\subsection*{Hard- and software}
	All our experiments were conducted on a Windows 10 Enterprise (version 1909) machine with an Intel Core i5-8500 6-Core 64-bit CPU and 16 GB RAM. We use Python 3.8.13 with Numpy 1.22.3, and Tensorflow~2.3.0 in CPU-only mode.
	
	\subsection*{Generation of random Knapsack instances}
	For a given value of $p^*$ we sample a set of items of total profit \mbox{$\sum p_i = p^*$} in the following way: the profit of the $i$-th item is always chosen uniformly at random among all integer values between $1$ and $p^*-\sum_{i'=1}^{i-1} p_{i'}$. This is repeated until all profits sum up to $p^*$. We chose this procedure in order to make it likely to have both, very profitable and less profitable items within one instance. Finally, we shuffle the order of the items. For each item, we then pick a size value uniformly in the interval $[0,1]$ and normalize these values such that their sum equals a uniformly at random chosen value $\sum s_i\in\left]1,2\right[$. We certainly want $\sum s_i>1$, because otherwise all items would fit into the Knapsack. On the other hand, $\sum s_i<2$ makes sense, because in our DP-NN (compare Section~\ref{Sec:DPNN}), we use $2$ as a placeholder for $+\infty$.
	
	\subsection*{Preparation of the training set} Since we can generate arbitrarily many random Knapsack instances, we use an infinite training set and never train on the same data point twice. A Knapsack instance with $n$ items yields $n$ training points, namely one for each step of the dynamic program. In order to avoid having the $n$ training points belonging to one instance in successive order, we generate training points belonging to several instances and shuffle them.
	
	\subsection*{Neural network architecture} For a given value $p^*$ and width $w$, the corresponding neural network used consists of an input layer with $p^*+2$ neurons (corresponding to the $p^*$ values of the previous dynamic programming state, as well as, the scalar profit and size values), three hidden layers with $w$ neurons each and ReLU activations, as well as an output layer with $p^*$ neurons (corresponding to the new state of the dynamic programming) without further activation function. That way, we choose the same depth as in the DP-NN (Section~\ref{Sec:DPNN}), but do not employ the specific knowledge about the widths of the three hidden layers. As in the DP-NN, each layer is not only connected to the previous layer, but also receives direct connections from the input layer. In total, by our results of Section~\ref{Sec:DPNN}, this architecture is theoretically able to exactly represent the dynamic programming transition function if $w\geq p^*(p^*-1)/2$.
	
	\subsection*{Training and validating a specific network} For a given value $p^*$ and width $w$, we train the neural network architecture described above as follows. We train in epochs of $1000$ batches with batch size $32$ using mean squared error (MSE) loss and the Adam optimizer, which is a robust standard choice. It makes sense to use MSE loss as it punishes errors in both directions equally hard and large errors harder than small errors. All other (hyper-)parameters are left to the default settings of Tensorflow, which empirically works quite well for our problem type and size. It takes between 8 and 30 seconds to train one epoch with our machine setup. We train until there are two successive epochs without improvement in training loss, which empirically happens after 10 to 80 epochs. Using a validation set that is randomly generated in the same way as the training set, we evaluate the trained network on 1000 batches of size 32 each. The resulting mean squared error is our final result for the given values of $p^*$ and~$w$.
	
	\subsection*{Finding the required width} For each of our three MSE thresholds and each value $p^*$, we always train networks with increasing widths $1,2,3, \dots$ as described above until a network achieves a validation MSE less or equal to the threshold.
	
	\subsection*{Seed generation} In order to ensure reproducibility of our experiments, each time before we train and validate an NN with given value $p^*$ and width $w$, we reset the random seeds of both Numpy and Tensorflow to $257\cdot p^* + w$, where $257$ is just an arbitrary prime number. Note that these packages only guarantee the same result of random experiments if the same package versions are used.
	
	
	\subsection*{Source Code} The source code is publicly available at \url{https://github.com/ChristophHertrich/neural-knapsack}. There, the file \texttt{README.md} explains how the code can be used to reproduce the experiments of this paper.
\end{APPENDIX}
%
%


\bibliographystyle{informs2014} 
\bibliography{knapsack} 


\end{document}